\documentclass{article}

\usepackage{graphicx}
\usepackage{subcaption}
\usepackage{booktabs} 
\usepackage{hyperref}
\usepackage{url}
\usepackage{xcolor}
\usepackage{plain}
\usepackage{geometry}
\usepackage{bbm}
\usepackage{color}
\usepackage{soul}
\usepackage{amsmath,amssymb,amsthm, mathtools}
\usepackage{algorithm}
\usepackage{mathrsfs}
\usepackage{algorithm, algorithmic}
\usepackage{cite}
\usepackage{bm}

\usepackage{color-edits}
\addauthor{sw}{blue}

\renewcommand{\cal}{\mathcal}
\newcommand\cA{{\mathcal A}}
\newcommand\cB{{\mathcal B}}
\newcommand{\cC}{{\cal C}}

\newcommand{\cY}{{\cal Y}}
\newcommand{\cD}{{\cal D}}
\newcommand{\cE}{{\cal E}}

\newcommand\cO{{\mathcal O}}

\newcommand{\cR}{{\mathcal R}}
\newcommand{\cS}{{\mathcal S}}

\newcommand{\cX}{{\mathcal X}}
\newcommand{\bE}{\mathbb{E}}

\newcommand{\bN}{\mathbb{N}}
\newcommand{\bP}{\mathbb{P}}

\newcommand{\bR}{{\mathbb R}}

\renewcommand{\leq}{\leqslant}
\renewcommand{\geq}{\geqslant}

\newcommand{\zhun}[1]{\textcolor{blue}{[ZD:~#1]}}

\usepackage{color}
\definecolor{english}{rgb}{0.0, 0.5, 0.0}

\newcommand{\p}{{\bm{p}}}
\newcommand{\br}{{\bm{r}}}
\newcommand{\bpi}{{\bm{\pi}}}

\newtheorem{theorem}{Theorem}[section]

\newtheorem{lemma}{Lemma}[section]

\newtheorem{remark}{Remark}[section]
\newtheorem{assumption}{Assumption}[section]

\newtheorem{definition}{Definition}[section]

\usepackage{algorithm, algorithmic}

\usepackage{hyperref}

\usepackage{float}
\title{Reinforcement Learning with Stepwise Fairness Constraints}

\author{Zhun Deng\thanks{Harvard University, \texttt{zhundeng@g.harvard.edu}.}
\quad He Sun\thanks{Harvard University, \texttt{he\_sun@g.harvard.edu}.}
\quad Zhiwei Steven Wu\thanks{Carnegie Mellon University, \texttt{zstevenwu@cmu.edu}.}
\quad Linjun Zhang\thanks{Rutgers University, \texttt{linjun.zhang@rutgers.edu}.}
\quad David C. Parkes\thanks{Harvard University, \texttt{parkes@eecs.harvard.edu}.}}
\date{}

\begin{document}

\maketitle

%

%

\begin{abstract}
AI methods are used in societally important settings, ranging from credit to employment to housing, and it is crucial to provide fairness in regard to algorithmic decision making. Moreover, many settings are dynamic, with populations responding to sequential decision policies.  We introduce the study of reinforcement learning (RL) with stepwise fairness constraints, requiring group fairness  at each time step. Our  focus is on tabular episodic RL, and we provide learning algorithms with strong theoretical guarantees in regard to policy optimality and fairness violation. Our framework provides useful tools to study the impact of fairness constraints in sequential settings and brings up new challenges in RL.
\end{abstract}

\section{Introduction}
Decision making systems trained with real-world data are deployed ubiquitously in our daily life, for example, in regard to credit, education, and medical care. However, those decision systems may demonstrate discrimination against disadvantaged groups due to the biases in the data \cite{Dwork2012}. In order to mitigate this issue, many have proposed to impose fairness constraints \cite{hardt2016equality,Dwork2012} on the decision, such that certain statistical parity properties are achieved. Despite the fact that fair learning has been extensively studied, most of this work is in the static setting without considering the sequential feedback effects of decisions. At the same time, in many scenarios, algorithmic decisions may incur changes in the underlying features or qualification status of individuals, which further feeds back to the decision making process; for example, banks' decision may induce borrowers to react, for example changing their FICO score by closing  credit cards. 

When there exist sequential feedback effects, even ignoring one-step feedback effects can harm minority groups~\cite{liu2018delayed}. In response,~\cite{tu2020fair} advocates to study a discrete-time sequential decision process, where responses to the decisions made at each time step are accompanied by changes in the features and qualifications of the population in the next time step. In particular, they study and show the drawback of myopic optimization together with requiring fairness at each time step, which we  refer to as \textit{stepwise fairness constraints}. While imposing stepwise fairness constraints is a natural way to incorporate fairness into a Markov decision process (MDP), it makes more sense from the perspective of a decision maker to consider  the long-term reward. Thus, in this paper, we study stepwise fairness constraints together with optimal, sequential-decision making, taking the perspective of a forward-looking decision maker. On one hand, our work could be viewed as a {\em Fair Partially Observable Markov Decision Process} (F-POMDP) framework to promote fair sequential decision making \footnote{The problem of long-term well-being of groups brought up in \cite{liu2018delayed} is due to a misalignment of the institute's reward function and well-being measurements, and cannot be solved in any existing literature. It also doesn't prevent stepwise fairness constraints from being a way to build fair MDP.}. On the other hand, our work also provides a computational tool for studying the quantities of interests, such as well-being of groups, in a natural sequential decision making setting.
%

In particular, we initiate both the theoretical and experimental studies of reinforcement learning, i.e. optimizing long-term reward in a partially observable Markov decision process (POMDP), under stepwise fairness constraints. We consider an episodic setting, which models for example 
economic and societal activities that exhibit  seasonality; e.g., new mortgage applicants who 
apply for loans from banks more often in the spring and summer season every year, 
or  graduate school admission, which  usually starts in the
autumn and completes around December every year. 
%
Similar to~\cite{liu2018delayed,tu2020fair}, we mainly consider two  types of fairness notions, \textit{demographic parity} and \textit{equalized opportunity}, and for a POMDP framework that has discrete actions and a discrete state space. These are illustrative of other stepwise fairness constraints that could be adopted. 
We take a model-based learning approach, and 
provide practical optimization algorithms that
enjoy strong theoretical guarantees  in regard to policy optimality and fairness violations
as the number of episodes increases. We summarize our contributions as below:

1. Theoretically, we demonstrate how to use sampled trajectories of individuals to solve RL with fairness constraints and provide theoretical guarantees to ensure vanishing regrets in reward and fairness violation as the number of episodes increases.

2. Experimentally, we implement the first algorithm for tabular episodic RL with stepwise fairness constraints.

\subsection{Related work}
As the rapid development of machine learning \cite{he2016deep,he2017mask,goodfellow2020generative,deng2021adversarial,deng2021improving,deng2021toward,deng2020towards}, there is an increasing interest in the study of decision making problems in the context of people\cite{Hardt2015,Shavit2019,Ball2019,Chen2020,sun2020decision}. Hardt et al.~\cite{Hardt2015}
model a classification problem as a sequential game (Stackelberg competition) between two players, where the first player has the ability to commit to his strategy before the second player responds. They characterize the equilibruim and obtain near optimal computationally efficient learning algorithms.  Shavit and Moses~\cite{Shavit2019} study an algorithmic decision-maker who incentivizes
people to act in certain ways to receive a better decision.
Ball et al.~\cite{Ball2019} study a model of predictive scoring, 
where there is a sender agent being scored, a receiver agent who 
wants to predict the quality of the sender, and an
intermediary who observes multiple, potentially mutable features of the sender.

There is also a growing literature on algorithmic fairness~\cite{liu2018delayed,Calders2009,Kusner2017,Dwork2012,deng2022fifa,deng2020representation,burhanpurkar2021scaffolding,zhang2022and,deng2020interpreting}.
%
Liu et al.~\cite{liu2018delayed}, for example, characterize the delayed impact of standard fairness criteria
under a feedback model with a single period of adaptation. They use a one step feedback model to capture the sequential dynamics of the environment. 
However, these papers do not consider the fairness in a more general sequential decision process. 
There is also  a line of literature regarding fair bandits, but not more general (PO)MDP problems~\cite{Joseph2016,Hashimoto2018}.
%
In regard to fairness considerations in reinforcement learning, it has also gained great attention~\cite{DAmour2020,Creager2019,Wen2021,Jabbari2017,Mandal2022,deng2021improving,kawaguchi2022understanding,deng2021toward}. In particular, Creager et al.~\cite{Creager2019} use  causal directed acyclic graphs as a unifying framework for fairness. D'Amour et al.~\cite{DAmour2020} use simulation to study the fairness of algorithms and  show that neither static nor single-step analyses is  enough to
 understand the long-term consequences of a decision system.
Jabbari et al.~\cite{Jabbari2017} 
define fairness constraints to require that an algorithm never prefers one action over another if the long-term reward of choosing the latter action is higher, whereas we consider groupwise notions of fairness. 
Mandal et al.~\cite{Mandal2022} adopt a welfare-based, axiomatic approach, and give a regret bound for the Nash Social, Minimum and generalized Gini Welfare. 
In contrast with our work, the fairness concepts are not group-based but rather based on the value contributed from different agents in the system. 
Similar to this paper, Zhang et al.~\cite{Zhang2020}
study the dynamics of population qualification and algorithmic decisions under
a partially observed Markov decision problem setting, but whereas they 
only consider myopic policies, we frame this as a
general reinforcement learning policies. 
%


\section{Preliminaries}


We consider a binary decision setting, with training examples that consist of triplets $(x, y, \vartheta)$, where $x\in\cX$ is a feature vector, $\vartheta\in \Lambda$ is a protected group attribute such as race or gender, and the label $y\in\{0,1\}$.
For simplicity, we only consider binary sensitive attributes $\Lambda=\{\alpha,\beta\}$, but our method can also be generalized to deal with multiple sensitive attributes (see Appendix~\ref{app:multipleatt}).
For $k\in \bN^+$, we use $[k]$ to denote the set $\{1,2,\cdots,k\}$. 

Based on feature $x$, a decision maker makes a decision $a\in\cA=\{0,1\}$ (e.g., make a loan or not). We will also denote an individual's state as $s=(x,y)$, and let $\cS=\cX\times\cY$.  After a decision is made, a possibly group-dependent reward, which may be stochastic, $r^\vartheta: (s,a)\mapsto \bR$ is obtained by the decision maker. 
A concrete  example of a reward function, with  $r_+,r_->0$,  is 
$$
r^\vartheta(s,a)=\left\{\begin{matrix}
r_+, &\mbox{if $y=1,a=1$;}\\ 
	-r_-,&\mbox{if $y=0,a=1$;}  \\
	0, &\mbox{if $a=0$.}
\end{matrix}\right.$$

Here, the decision maker gains $r_+$ by accepting a qualified individual and incurs a cost $r_-$ by accepting an unqualified individual. 

\subsection{Sequential setting}


%

Our model as a partially observable Markov decision process (POMDP) mainly follows~\cite{tu2020fair}, but they consider a fair, myopically-optimizing policy while we consider long-term rewards as in typical RL settings. 
Following~\cite{liu2018delayed,tu2020fair}, the decision maker is
interested in the expected reward 
achieved across time for a \textit{\textbf{random individual}} drawn from the population.
Each random individual has their group membership sampled according to 
$p_\alpha=\bP(\vartheta=\alpha)$ and $p_\beta=\bP(\vartheta=\beta)$, and interacts with the decision maker over multiple time steps.
%
%
At each time step $h$, the sampled individual with attribute $\vartheta$ has feature $x^\vartheta_h=x^\vartheta \in\cX$ along with a hidden qualification status $y^\vartheta_h=y^\vartheta \in\{0,1\}$. One example is that the feature $x^\vartheta_h$ is determined by the hidden qualification status with $x^\vartheta_h\sim p(\cdot|y^\vartheta)$. 
We call $s^\vartheta_h=(x^\vartheta_h,y^\vartheta_h)$ the 
{\em state } 
of the individual at time $h$. 
The initial state $s^\vartheta_1$ is sampled from $p^\vartheta$.

At each time step $h$, the 
decision maker adopts a decision $a^\vartheta_h$  based on the observed feature $x^\vartheta$ by following a {\em group-dependent policy} $\pi^\vartheta_h(x^\vartheta)$, i.e. $a^\vartheta_h\sim \pi^\vartheta_h(x^\vartheta_h)$, where $\pi^\vartheta_h:\cX\rightarrow\Delta(\cA)$, and $\Delta(\cA)$ is the set of distributions on $\cA$.\footnote{Following~\cite{tu2020fair}, 
we use group-dependent policies so that the formulation can be more generalized. Our technique can also be used for group-independent policies if this is required.}
%
%
The decision maker receives 
{\em reward} $r^\vartheta(s^\vartheta_h,a^\vartheta_h)$, and $r^\vartheta\in[l,u]$, where $-\infty<l-u< \infty$. Without loss of generality, let us assume $r^{\vartheta}\in[0,1]$. The mean of stochastic reward function is denoted by $r^{*\vartheta}$.

After the decision is made, the individual is informed of the decision and their qualification status,  and  their qualification status, $y^\vartheta_{h+1}$, and features $x^\vartheta_{h+1}$, may
then undergo a stochastic transition.
 We assume that this transition follows
 a {\em time-invariant} but {\em group-dependent transition kernel},
which we denote as 
$p^{*\vartheta}(s^\vartheta_{h+1}|s^\vartheta_h,a^\vartheta_h)$, where $p^{*\vartheta}:\cS\times \cA\mapsto \Delta(\cS)$, and $\Delta(\cS)$ is the set of distributions on $\cS$. 
 As explained in~\cite{tu2020fair}, in addition to thinking about a single, randomly chosen individual 
 repeatedly interacting with the decision maker, this also immediately models a finite population of randomly chosen individuals, some from each group, and with all individuals in a group
 subject to the same, group-contingent decision policies. In addition, we have a reward function pair, $(r^{\alpha},r^{\beta})$, 
 which may be stochastic. 





%
%

 \subsection{Fair policies}
 
 The goal of the decision maker is to   find a {\em policy}, $\bpi=(\pi^\alpha,\pi^\beta)$, to
 maximize the total, expected reward over an episode   for a random individual, while satisfying stepwise fairness constraints, i.e. imposing a certain type of fairness constraint on states and actions for each time step.
  
  
A random individual   in group $\vartheta$, and with decision policy $\pi^\vartheta$,  follows
stochastic state, action, reward sequence $s^\vartheta_1,a^\vartheta_1,r^\vartheta(s^\vartheta_1,a^\vartheta_1),s^\vartheta_2,a^\vartheta_2,r^\vartheta(s^\vartheta_2,a^\vartheta_2),s^\vartheta_3,\cdots$.
%
%
%
%
%
Let $\bE^{\pi,p}[\cdot],\bP^{\pi,p}[\cdot]$ be the expectation and probability of a random variable defined with respect to this stochastic process.
%
%
We denote the expected reward for a random individual at time step $h$  as,
$$\cR^*_h(\p^*,\bpi)=\sum_{\vartheta\in\{\alpha,\beta\}}p_\vartheta\cdot  \bE^{\pi^\vartheta,p^{*\vartheta}}[r^{*\vartheta}(s^\vartheta_h,a^\vartheta_h)].$$

Note here that $\bE^{\pi^\vartheta,p^{*\vartheta}}[\cdot]$ refers to the expected value for an individual in group $\vartheta$, i.e.,
conditioned on the individual being sampled in this group.
Our ideal goal is to obtain the  policy  $\bpi^*$ that solves the following optimization problem:
\begin{align*}
&\qquad\quad\max_{\bpi}\sum_{h=1}^H \cR^*_h (\p^*,\bpi)\\
  &\qquad\qquad \mbox{s.t.}~~\forall h\in[H],  \\
  &\mathit{faircon}(\{\pi^\vartheta,p^{*\vartheta},s^\vartheta_h,a^\vartheta_h\}_{\vartheta=\{\alpha,\beta\}}),
\end{align*}
%
where $\mathit{faircon}$ corresponds to a particular fairness concept.
We consider two group fairness definitions, and also discuss how the approach can be  extended to additional fairness concepts (see Section~\ref{subsec:extfairnotion}).
Specifically, we consider the following two   fairness concepts: 

(i) \textit{RL with demographic parity (DP)}. For this case,  $\mathit{faircon}$ is
   $$\quad\bP^{\pi^\alpha,p^{*\alpha}}[a^\alpha_h=1]=\bP^{\pi^\beta,p^{*\beta}}[a^\beta_h=1],$$
   which means at each time step $h$, the decision for individuals from different groups is  statistically independent of the sensitive attribute.
   
(ii) \textit{RL with equalized opportunity (EqOpt)}. For this case, $\mathit{faircon}$ is
    $$\bP^{\pi^\alpha,p^{*\alpha}}[a^\alpha_h=1|y^\alpha_h=1]=\bP^{\pi^\beta,p^{*\beta}}[a^\beta_h=1|y^\beta_h=1],$$
which means at each time step $h$, the decision for a random individual from
each of the two different groups is  statistically independent of the sensitive attribute conditioned 
on the individual being qualified.
In each case, $\bP^{\pi^\alpha,p^{*\alpha}}[\cdot]$ refers to the probability for an individual in group $\vartheta$, i.e.,
conditioned on the individual being sampled in this group.

The above optimization problems are feasible under technical Assumption~\ref{ass:kernel}, presented later, and we assume feasibility throughout the paper (see also Appendix~\ref{app:feasibility}). 

\begin{remark}
Since we are modelling the behavior of a randomly drawn individual from the population, the objective  should be viewed as to find a policy pair $\bpi=(\pi^\alpha,\pi^\beta)$ to optimize the long-term reward of 
of the decision maker while ensuring fairness over the population. 
\end{remark}

\subsection{Episodic RL protocol} 
This is a learning setting, and we study an
episodic sequential decision setting where
a learner  repeatedly   interacts with an environment
across $K>0$  independent episodes.
Such a scenario is natural in a number of practical settings, as we discussed in the introduction.
%
%
Throughout the paper, we consider the tabular case, i.e., we assume finite cardinality for $\cS$ and $\cA$.\footnote{In the example of credit score and loan payment in \cite{liu2018delayed}, the credit score is discretized and served as $\cX$ here and the action space $\cA$ and qualification status space $\cY$ are both $\{0,1\}$.}
%
Without loss of generality, we can assume 
that the initial state of an episode is a fixed state $s_0$ (the next state can be sampled randomly, and  state $s_0$ does not contribute to any reward or fairness considerations, see \cite{brantley2020constrained} for further detailed explanation).
%
%
%
%
At episode $k\in[K]$, denote policy pair $\bpi_k=(\pi^\alpha_k,\pi^\beta_k)=\{(\pi^\alpha_{k,h},\pi^\beta_{k,h})\}_{h=1}^H$, where $H$ is the horizon.
%
An individual sampled from group $G_\vartheta$ starts from state $s^\vartheta_{k,1}$, thus, we can consider starting state pair $(s^\alpha_{k,1}, s^\beta_{k,1})=(s^\alpha_0,s^\beta_0)$ for the trajectory of different groups (the initial state depends on the group from which the individual is drawn). At each time step $h\in [H]$,
the decision maker selects an action $a^\vartheta_{k,h}\sim \pi^\vartheta_{k,h}(x^\vartheta_{k,h})$.
Here, although the policy only uses the $x$ component, it is convenient to write it as a function of $s$.
The decision maker gets  reward $r^\vartheta(s^\vartheta_{k,h},a^\vartheta_{k,h})$, and the state of the individual for next time step is drawn according to $s^\vartheta_{k,h+1}\sim p^{*\vartheta}(\cdot|s^\vartheta_{k,h},a^\vartheta_{k,h})$.

\begin{remark}
The policy is only based on the  \textit{\textbf{feature vector}} $x$. However, we are able to access $y$ in the training data.
\end{remark}

\section{Learning Algorithms}\label{sec:EFH algo}

Before we formally state our algorithms, we need to gather data to estimate the unknown quantities such as reward function mean pair $\br^*=(r^{*\alpha},r^{*\beta})$ and transition probability pair $\p^*$. In addition, we will incorporate an {\em exploration bonus} to further modify the estimation.

\paragraph{Data gathering and estimation.} 
In order to analyze the policy effect on the population, we model the behavior of a randomly drawn individual who interacts with the environment across $H$ steps. Here, we demonstrate how to aggregate individuals' data   for each episode and estimate quantities of interest. Specifically, at episode $k\in[K]$, for each group $\vartheta$, we assume  $n^\vartheta_{k}$  individuals are drawn, according to $p_\alpha$ and $p_\beta$. Throughout the paper, we assume $n^\vartheta_k\ge 1$, for each $\vartheta$ and $k$. 
A decision is made independently for each individual at each step, using a group-specific policy, leading to a stochastic transition in the state of the individual. In Appendix~\ref{app:optin}, we will further discuss how to gather data when allowing individuals who opt in or out 
during an episode. We use the counting method to obtain estimates of the statistics of interest.  For the $i$-th individual in episode $k$, their status and action at time step $h$ is denoted as $s^{\vartheta,(i)}_{k,h}$ and $a^{\vartheta,(i)}_{k,h}$. 

For $\vartheta\in\{\alpha,\beta\}$, let $\p_k=(p^\alpha_k,p^\beta_k)$ and $\br_k=(r^\alpha_k,r^\beta_k)$, 
\begin{align*}
N^\vartheta_{k}(s,a)&=\max\big\{1, \sum_{t\in [k-1],h\in[H],i\in[n^\vartheta_k]} \mathbbm{1}\{s^{\vartheta,(i)}_{k,h}=s,a^{\vartheta,(i)}_{k,h}=a\}\big\},\\
p^\vartheta_{k} (s'|s,a)&=\frac{1}{N^\vartheta_{k}(s,a)}\sum_{t\in [k-1],h\in[H],i\in[n^\vartheta_k]} \mathbbm{1}\{s^{\vartheta,(i)}_{k,h}=s,a^{\vartheta,(i)}_{k,h}=a, s^{\vartheta,(i)}_{k,h+1}=s'\},\\
\hat{r}^\vartheta_{k} (s,a)&=\frac{1}{N^\vartheta_{k}(s,a)}\sum_{t\in [k-1],h\in[H],i\in[n^\vartheta_k]}r^\vartheta(s,a)\mathbbm{1}\{s^{\vartheta,(i)}_{k,h}=s,a^{\vartheta,(i)}_{k,h}=a\}.
\end{align*}

\paragraph{Exploration bonus method.} In RL, it is standard to introduce optimism in order to encourage exploring under-explored states. Specifically, for $\vartheta\in\{\alpha,\beta\}$, we adopt a {\em bonus term}, $\hat{b}^\vartheta_k$, to add to the estimated reward function $\hat{r}^\vartheta_{k}$, such that we obtain $r^\vartheta_{k}(s,a)=\hat{r}^\vartheta_{k}(s,a)+\hat{b}^\vartheta_{k}(s,a)$, where the $\hat{b}^\vartheta_{k}(s,a)$ values assign larger values for under-explored $(s,a)$'s. We  specify how to choose $\hat b^\vartheta_k$ in Section \ref{subsec:theory}, and  denote
\begin{equation*}
	\cR_{k,h}(\p,\bpi)= \sum_{\vartheta\in\{\alpha,\beta\}} p_\vartheta\cdot \bE^{\pi^\vartheta,p^{\vartheta}}[r_k^{\vartheta}(s^\vartheta_{k,h},a^\vartheta_{k,h})].
\end{equation*}

For the purpose of analysis, we  treat $p_\vartheta$'s  as known constants for simplicity; for example, perhaps these proportions  are provided by census.

\textbf{Practical optimization for DP.} In reality, given we don't have access to $\p^*$ and $r^{*\vartheta}$, we need to solving a surrogate optimization problem and hope the optimal policy can have similar performance as the ideal optimal policy under certain performance criteria. In the following, we provide a simple algorithm for RL with demographic parity. It is based on optimization under $p^{\vartheta}_k$ and $r^{\vartheta}_k$:
\begin{align*}
  &\qquad\quad\max_{\bpi\in\Pi_k}\sum_{h=1}^H \cR_{k,h} (\p_k,\bpi),\\
   &\qquad\qquad~~s.t.~~\forall h\in[H],\\
   &|\bP^{\pi^\alpha,p_k^{\alpha}}(a^\alpha_{k,h}=1)-\bP^{\pi^\beta,p_k^{\beta}}(a^\beta_{k,h}=1)|\le \hat c_{k,h},
\end{align*}
where we have $\Pi_k=\{(\pi^\alpha,\pi^\beta): \pi^\vartheta(a=1|x)\ge \eta^{\mathit{DP}}_k,\forall x,a,h,\vartheta\}.$ $\Pi_k$ can ensure the reachability from any $x$ to the decision $a=1$. Here $\{\eta^{\mathit{DP}}_k\}_k$ is a sequence of real numbers and $\{\hat c_{k,h}\}_{k,h}$ are relaxations.
Intuitively, if we set $\eta^{\mathit{DP}}_k$ and $\hat c_{k,h}$ to be decreasing and vanishing as $k$ increases, we would expect the above optimization  to approach the ideal optimization problem as $k$ increases. We formalize this in Section \ref{subsec:theory}. 


\paragraph{Practical optimization for EqOpt.} For equalized opportunity, and similar to the case of DP, we have
 \begin{align*}
  &\qquad\qquad\max_{\pi\in\Pi_k}\sum_{h=1}^H \cR_{k,h} (\p_k,\bpi),\\
  &s.t.~~\forall h\in[H],~~|\bP^{\pi^\alpha,p_k^{\alpha}}(a^\alpha_{k,h}=1|y^\alpha_{k,h}=1)-\bP^{\pi^\beta,p_k^{\beta}}(a^\beta_{k,h}=1|y^\beta_{k,h}=1)|\le \hat d_{k,h},
 \end{align*}
where  we have $\Pi_k=\{(\pi^\alpha,\pi^\beta): \pi^\vartheta(a=1|x)\ge \eta^{\mathit{EqOpt}}_k,\forall x,a,h,\vartheta\}.$
Here $\{\eta^{\mathit{EqOpt}}_k\}_k$ is a sequence of real numbers and $\{\hat d_{k,h}\}_{k,h}$ are relaxations. We formalize this in Section \ref{subsec:theory}.   

\paragraph{Algorithm.} We can solve these optimization problems through  occupancy measures, and they each become different kinds of quadratically constrained linear programs (QCLP) (see Appendix~\ref{app:occupancy}). 
Although QCLP is generally NP-hard, many methods based on relaxations and approximations such as semi-definite program (SDP) have been extensively discussed. We use Gurobi to solve these relaxed optimization problems.

\section{Theoretical Analysis}


In order to track the performance, we consider the following regrets. For policy pairs $\{\bpi_k\}_{k=1}^K$, for {\em reward regret} in episode $k$, we  track:
$$\cR^{\text{type}}_{\mathrm{reg}}(k)= \frac{1}{H}\sum_{h=1}^H\Big(\cR^*_h (\p^*,\bpi^{*\text{type}})-\sum_{t=1}^k\cR^*_h (\p^*,\bpi_k)\Big),$$
where $\bpi^{*\text{type}}$ is the optimal policy pair of RL with constraint types mentioned above and $\text{type}\in\{\mathit{\mathit{DP}},\mathit{\mathit{EqOpt}}\}$. For simplicity, we will omit the supscript ``type" when it is clear from the context and use $\bpi^*=(\pi^{*\alpha},\pi^{*\beta})$. 

For the fairness constraints, we consider the violation for each type of constraint in episode $k$ as the following,
$$\cC_{\mathrm{reg}}^{\mathit{DP}}(k)=\frac{1}{H}\sum_{h=1}^H \big|\bP^{\pi^{\alpha}_k,p^{*\alpha}}(a^\alpha_{k,h}=1)-\bP^{\pi^\beta_k,p^{*\beta}}(a^\beta_{k,h}=1)\big|,$$
and 
\begin{align*}
 \cC_{\mathrm{\mathrm{reg}}}^{\mathit{\mathit{EqOpt}}}(k)=\frac{1}{H}\sum_{h=1}^H \big|\bP^{\pi^{\alpha}_k,p^{*\alpha}}(a^\alpha_{k,h}=1\big|y^\alpha_{k,h}=1)-\bP^{\pi^\beta_k,p^{*\beta}}(a^\beta_{k,h}=1|y^\beta_{k,h}=1)\big|.  
\end{align*}
The theoretical guarantees hold for any episode, not just the last episode.

\subsection{Choice of various of quantities} \label{subsec:theory}
In this part, we provide a formal theoretical guarantee for the performance of our algorithms under the previously mentioned criteria with suitably chosen quantities $\{\hat b^\vartheta_k\}_k$, $\{\hat c_k\}_k$, and $\{\hat d_k\}_k$, for each episode $k$. 

\paragraph{$Q$ and $V$ functions.} Two of the mostly used concepts in RL are $Q$ and $V$ functions. Specifically, $Q$ functions track the expected reward when a learner starts from state $s\in\cS$. Meanwhile, $V$ functions are the corresponding expected $Q$ functions of the selected action. For a reward function $r$ and a transition function $p$, the $Q$ and $V$ functions are defined as:
\begin{align*}
Q^{\pi,p}_r(s,a,h)&=r(s,a)+\sum_{s'\in\cS}p(s'|s,a)V^{\pi,p}_r(s',h+1),\\
V^{\pi,p}_r(s,h)&=\bE_{a\sim\pi(\cdot|s)}[Q^{\pi,p}_r(s,a,h)],
\end{align*}
where we set $V^{\pi,p}_r(s,H+1)=0$.

\paragraph{Choice of $\hat b^\vartheta_k$.} For $\{\hat b^\vartheta_k\}_k$, similar to \cite{brantley2020constrained}, we need $\{\hat b^\vartheta_k\}_k$ to be valid. 
\begin{definition}[Validity]
A bonus $\hat b^\vartheta_k$ is valid for episode $k$ if for all $(s,a)\in\cS\times\cA$ and $h\in[H]$,
\begin{align*}
&\Big |\hat r^\vartheta_k(s,a)-r^{*\vartheta}_k(s,a)+\sum_{s'\in\cS}\Big( p^\vartheta_k(s'|s,a)- p^{*\vartheta}(s'|s,a)\Big)
V^{\pi^{*\vartheta},p^{*\vartheta}}_{r^{*\vartheta}}(s',h+1)\Big|\le \hat b^\vartheta_k(s,a).
\end{align*}
\end{definition}

Following the exporation-bonus setting~\cite{brantley2020constrained}, we set $\hat b^\vartheta_k=\min\Big\{2H,2H\sqrt{\frac{2\ln(16SAHk^2/\delta)}{N^\vartheta_k(s,a)}}\Big\}$, and have the following lemma.
\begin{lemma}\label{lm:b_k}
With probability at least $1-\delta$, for
$$\hat{b}^\vartheta_k(s,a)=\min\Big\{2H,2H\sqrt{\frac{2\ln(16SAHk^2/\delta)}{N^\vartheta_k(s,a)}}\Big\},$$ 
the bonus $\hat b^{\vartheta}_k(s,a)$ is valid for every  episode $k$ simultaneously.
\end{lemma}

\paragraph{Choice of $\hat c_{k,h}$ and $\hat d_{k,h}$.} For $\{\hat c_{k,h}\}_{k,h}$ and $\{\hat d_{k,h}\}_{k,h}$, we require them to be \textit{compatible}. 
\begin{definition}[Compatibility]
The sequence $\{\hat c_{k,h}\}_{k,h}$ is {\em compatible} if for all $h\in[H], k\in[K],$ $$|\bP^{\pi^{*\alpha},p_k^{\alpha}}(a_h=1)-\bP^{\pi^{*\beta},p_k^{\beta}}(a_h=1)|\le \hat c_{k,h}.$$
%
The sequence $\{\hat d_{k,h}\}_{k,h}$ is {\em compatible} if for all $h\in[H], k\in[K]$, $$|\bP^{\pi^{*\alpha},p_k^{\alpha}}(a_h=1|y_h=1)-\bP^{\pi^{*\beta},p_k^{\beta}}(a_h=1|y_h=1)|\le \hat d_{k,h}.$$
\end{definition}

Briefly speaking, we hope that when substituting $p^{*\vartheta}$ to $p^{\vartheta}_k$, that $\hat c_{k,h}$ and $\hat d_{k,h}$ can control the fairness constraints violation. 
Let us use $S$ to denote $|\cS|$ and $A$ to denote $|\cA|$.
\begin{lemma}\label{lm:c_kd_k}
Denote $N^{\vartheta,\min}_{k}=\min_{s,a}N^\vartheta_k(s,a)$. For any $\{\epsilon_k\}_{k=1}^K$, with probability at least $1-\delta$, we take
$$\hat c_{k,h}=\sum_{\vartheta\in\{\alpha,\beta\}}H\sqrt{\frac{2S\ln(16SAHk^2/(\epsilon_k\delta))}{N^{\vartheta,\min}_{k}}}+2\epsilon_k HS.$$
Then, the sequence $\{\hat c_{k,h}\}_{k,h}$ is compatible.
\end{lemma}
Similarly, for $\hat d_{k,h}$, we have the following lemma. 
\begin{lemma}\label{lm:d}
Denote $p^{\vartheta,\min}_{k}=\min_{s,a}p^\vartheta_k(y=1|s,a)$. For any $\{\epsilon_k\}_{k=1}^K$, with probability at least $1-\delta$, we take
$$\hat d_{k,h}=\sum_{\vartheta\in\{\alpha,\beta\}}\frac{3H\sqrt{\frac{2S\ln(32SAk^2/(\epsilon_k\delta))}{N^{\vartheta,\min}_{k}}}+3\epsilon_k HS}{p^{\vartheta,\min}_{k}\left(p^{\vartheta,\min}_{k}-\sqrt{\frac{4\ln 2+2\ln(4SAk^2/\delta)}{N^{\vartheta,\min}_{k}}}\right)}$$
if $p^{\vartheta,\min}_{k}>\sqrt{\frac{4\ln 2+2\ln(4SAk^2/\delta)}{N^{\vartheta,\min}_{k}}}$; Otherwise, we set $\hat d_{k,h}=1$.
Then, the sequence $\{\hat d_{k,h}\}_{k,h}$ is compatible.
\end{lemma}


\subsection{Main theorems}

In this subsection, we provide our formal theoretical guarantees for the reward regret and fairness constraints violation. We require technical Assumption~\ref{ass:kernel}.
%
%
\begin{assumption}\label{ass:kernel}
(a). For all $(s,a,s')\in\cS\times\cA\times\cS$, there exists a universal constant $C>0$, such that $p^{*\vartheta}(s'|s,a)\ge C$ for $\vartheta=\{\alpha,\beta\}$. (b). For all $(x,a)\in\cX\times\cA$, there exists a universal constant $\tilde C$, such that $\pi^{*\vartheta}(a|x)\ge \tilde C$.
\end{assumption}
Assumption~\ref{ass:kernel}  implies irreducibility of the Markov process and  ensures feasibility of our
optimization problems (see Appendix~\ref{app:feasibility}). 
Recall that at episode $k\in[K]$, for each group $\vartheta$, we have $n^\vartheta_{k}\ge 1$ individuals drawn for each group $\vartheta$. 
And as a concrete exemplified choice of $\eta_k$ and $\epsilon_k$, we take $\eta_k=k^{-\frac{1}{3}}$ and $\epsilon_k=(kHS)^{-1}$. 

\paragraph{Reward regret.} For the reward regret, for either demographic parity or equalized opportunity, we can provide the following theoretical guarantee. Recall for two positive sequences $\{a_j\}$ and $\{b_j\}$, we write $a_j =\cO(b_j)$ if $\lim_{j\rightarrow\infty}(a_j/b_j) < \infty$. 

\begin{theorem}\label{thm:reward}
For $type\in\{\mathit{DP},\mathit{EqOpt}\}$, with probability at least $1-\delta$, there exists a threshold $T=\cO\left(\Big(\frac{H\ln(SA/\delta)}{n_k}\Big)^{3}\right)$, such that for all $k\ge T$,
$$\cR^{\text{type}}_{\mathrm{reg}}(k)=\cO\left(Hk^{-\frac{1}{3}}\sqrt{HS\ln(S^2AH^2k^3/\delta)}\right).$$
\end{theorem}

By Theorem~\ref{thm:reward}, our algorithms for each of the group fairness notions can ensure vanishing reward regrets when $k$ goes to infinity, which implies that the performance in regard to regret reward improves as the number of episodes increases.

\paragraph{Fairness constraints violation.} For fairness violation, we have the following Theorem~\ref{thm:constraint}.
\begin{theorem}\label{thm:constraint}
For $type\in\{\mathit{DP},\mathit{EqOpt}\}$, with probability at least $1-\delta$, there exists a threshold $T=\cO\left(\Big(\frac{H\ln(SA/\delta)}{n_k}\Big)^{3}\right)$, such that for all $k\ge T$,
$$\cC_{\mathrm{reg}}^{\text{type}}(k)\le \cO\left( k^{-\frac{1}{3}}\sqrt{SH\ln(S^2HAk^3/\delta)}\right).$$
\end{theorem}

By Theorem~\ref{thm:constraint},  our algorithms for each of the group fairness notions can ensure vanishing fairness violation when $k$ goes to infinity, which implies the performance in regard to fairness violations improves as  the number of episodes increases.

\section{Experiments}
\paragraph{Settings.}
We take $H=8$ for each episode, and update our policy every 
$k=2^{l}$ episodes, where $l=3,4,\ldots, 18$. We update our policy in this non-linear way  for reasons of computational cost, i.e., to reduce the number of optimization problems we need to solve while still collecting a large quantity of data. 
We choose the relaxation parameters $\hat{c}_{k,h}$ and $\hat{d}_{k,h}$ as defined in the previous sections. 
After each policy update, we use 8,000 episodes to evaluate the  new policy. We also produce confidence intervals by repeating each experiment five
times.


\paragraph{Estimation and Optimization Process.}

%
\begin{figure*}[h!]
\centering

\begin{subfigure}[b]{0.32\textwidth}
\centering
\includegraphics[width = \textwidth]{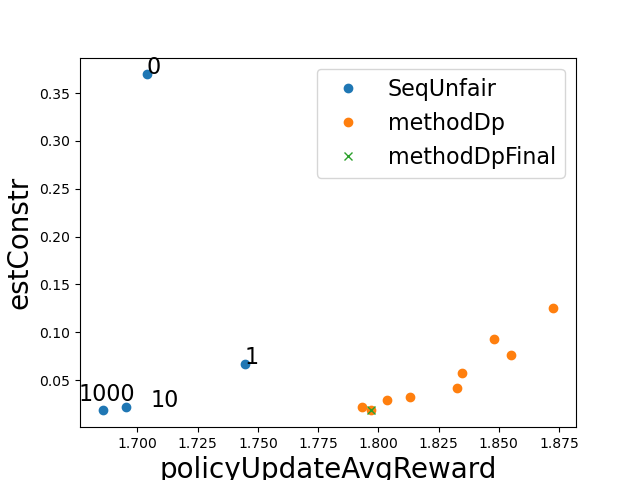}
\caption{DP FICO Pareto ~~~~~~~~~ \label{DpRealGenerativeMix_linear_estConstr_policyUpdateAvgReward}}
\end{subfigure}
\hfill
\begin{subfigure}[b]{0.32\textwidth}
\centering
\includegraphics[width = 1\textwidth]{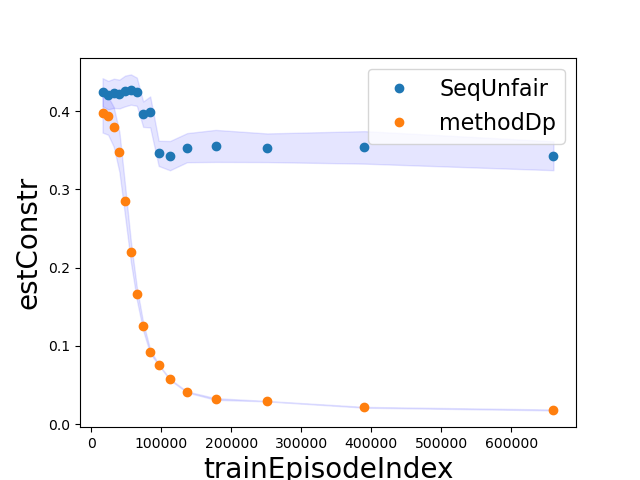}
\caption{DP FICO estConstr \label{DpRealGenerativeMix_linear_estConstr_trainEpisodeIndex}}
\end{subfigure}
\hfill
\begin{subfigure}[b]{0.32\textwidth}
\centering
\includegraphics[width = 1\textwidth]{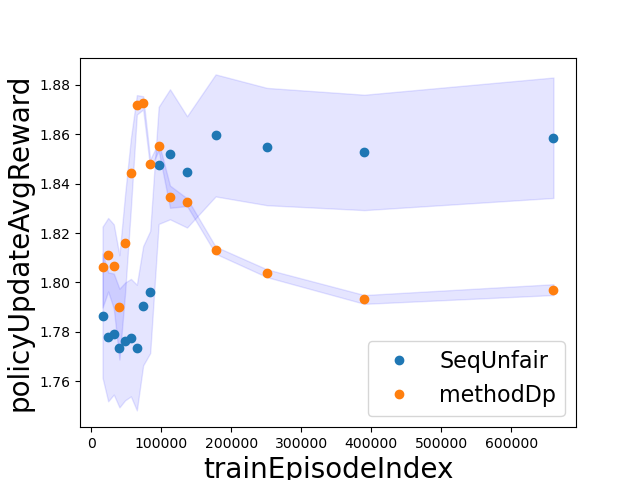}
\caption{DP FICO return \label{DpRealGenerativeMix_linear_policyUpdateAvgReward_trainEpisodeIndex}}
\end{subfigure}
\vskip\baselineskip
\begin{subfigure}[b]{0.32\textwidth}
\centering
\includegraphics[width = \textwidth]{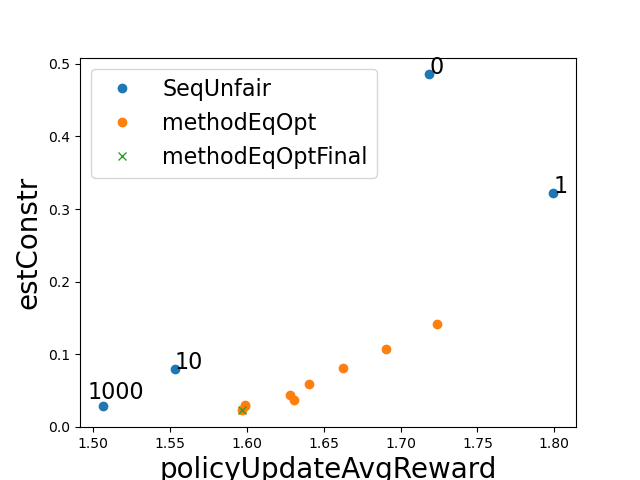}
\caption{EqOpt FICO Pareto ~~~~~~~~~ \label{EqOptRealGenerativeMix_linear_estConstr_policyUpdateAvgReward}}
\end{subfigure}
\hfill
\begin{subfigure}[b]{0.32\textwidth}
\centering
\includegraphics[width = 1\textwidth]{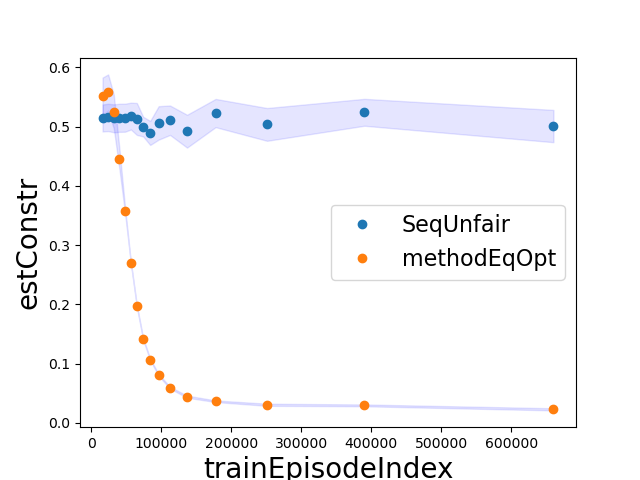}
\caption{EqOpt FICO estConstr ~~~\label{EqOptRealGenerativeMix_linear_estConstr_trainEpisodeIndex}}
\end{subfigure}
\hfill
\begin{subfigure}[b]{0.32\textwidth}
\centering
\includegraphics[width = 1\textwidth]{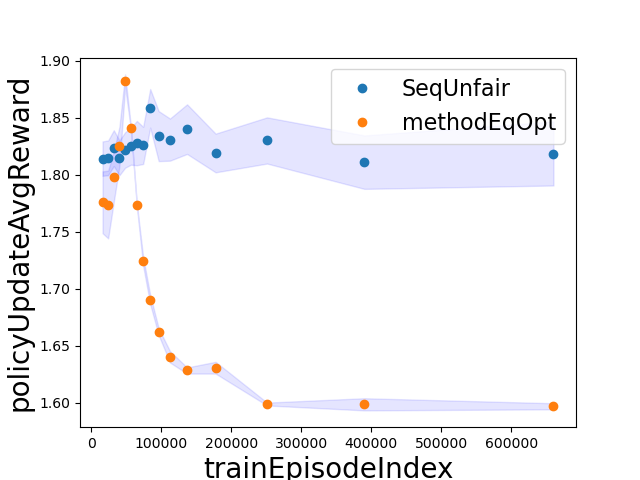}
\caption{EqOpt FICO return \label{EqOptRealGenerativeMix_linear_policyUpdateAvgReward_trainEpisodeIndex}}
\end{subfigure}
$$\vspace{-0.5in}$$
\caption{FICO data result. \ref{DpRealGenerativeMix_linear_estConstr_policyUpdateAvgReward} and \ref{EqOptRealGenerativeMix_linear_estConstr_policyUpdateAvgReward} for Pareto frontier, \ref{DpRealGenerativeMix_linear_estConstr_trainEpisodeIndex}\ref{DpRealGenerativeMix_linear_policyUpdateAvgReward_trainEpisodeIndex} for the constraint violation level in the training dynamics  \ref{EqOptRealGenerativeMix_linear_estConstr_trainEpisodeIndex} \ref{EqOptRealGenerativeMix_linear_policyUpdateAvgReward_trainEpisodeIndex} for the avergage episodic return over the training dynamics. In Pareto plots, the point with cross marker is the result for the final episode. The text close to points stand for penalty parameters}
\end{figure*}

We estimate transition probabilities and rewards using the counting method outlined above. These estimates are used as the input for our algorithm.
The optimization problems are non-convex, and the detailed optimization formulations are included in Appendix~\ref{app:experiment}. We use Gurobi to solve, and set the optimality-value tolerance, feasibility tolerance, and solving time limit as 1e-3, 1e-6, and 300 seconds, respectively. We use the barrier algorithm for all problems. 

\paragraph{FICO Data.} We adapt FICO score data\footnote{\url{https://docs.responsibly.ai/notebooks/demo-fico-analysis.html}} to form
our data generating process. The FICO score dataset contains data from non-Hispanic White and Black cohorts consist of $j \in \cX=\{0, \ldots, 4\}$, which represent normalized score $\{0, 25, 50, 75, 100\}$ and can be viewed as the feature. In particular, the FICO data provides empirical  distributions for credit scores of different sensitive groups, i.e. $\hat{\bP}_{\text{FICO}}(x^\vartheta = j)$,  along with an empirical qualification distribution conditioned on each score level $\hat{\bP}_{\text{FICO}}(y^\vartheta = y|x^\vartheta = j)$, where $y\in\{0,1\}$. 
%
We simulate the data generating process according to the empirical distributions stated above. For the population dynamics,  
we model the intial score distribution as
$\bP(x^\vartheta_0=j) = \hat{\bP}_{\text{FICO}}(x^\vartheta_0=j)$, according to
 the empirical  FICO distribution. 
For the initial qualification distribution conditioned on  score, 
$\bP(y^\vartheta_0=1|x_0=j)=\hat{\bP}_{\text{FICO}}(y^\vartheta = 1|x^\vartheta = j)$.
%
Then, we generate the underlying group-dependent and time invariant transition kernel $p^{*\vartheta}$ in the following way: we first set a distribution for $p^{*\vartheta}(x'=j'|x=j,y=w,a=v)$ for $j',j\in\cX$ and $w,v\in\{0,1\}$. Then, we set $p^{*\vartheta}(x',y'|x,y,a)$ as $p^{*\vartheta}(x'|x,y,a)\hat{\bP}_{\text{FICO}}(y^\vartheta = y'|x^\vartheta = x').$
More details about the data generating process is specified in Appendix~\ref{app:experiment}.

\paragraph{Reward function.} We choose a score-dependent reward, conditioned on the qualification level and decision as the following:
%
\[
    r^\vartheta(x_h, y_h, a_h) =  
\begin{cases}
    \beta^\vartheta_1 x_h , & \text{if } y_h = 1, a_h = 1;\\
    -\beta^\vartheta_2 x_h , & \text{if } y_h = 0, a_h = 1;\\
    0,                    & \text{if } a_h = 0,
\end{cases}
\]
where $\beta^\vartheta_1, \beta^\vartheta_2 \in (0,1)$. Our reward function is chosen based on the following considerations: in real life, the decision maker may give a higher loan amount to a candidate with a higher credit score. As a result, the decision maker benefits more when a qualified candidate  has higher score, i.e. higher reward value (e.g. higher total repayment interests) gained from a qualified candidate with higher score because of higher loan amount; the decision maker also suffers more (more negative reward gained) if the candidate with higher score is not qualified, because a larger amount of loan is not repaid. In our experiment, we set $ \beta^\alpha_1=0.1, \beta^\beta_1=0.9, \beta^\alpha_2=0.9, \beta^\beta_2=0.1$. 
This difference in reward function per group further reflects a decision maker who may,  potentially irrationally and  unfairly, benefit or penalize through their reward function for different groups, even for two individuals who otherwise have the same score $x$ and qualification $y$ (perhaps the decision maker worries that the model is not equally accurate or calibrated per group).



\paragraph{Baselines.}

We use the following policies as baselines:

1. Sequentially optimal policies without fairness constraints.
    %

2. Sequentially optimal policies with an objective that includes a fairness penalty term, which serves as an alternative method for our proposed optimization problem. 
   In this case, our optimization objective, for a particular kind of fairness constraint, is:
    $$\max_{\bpi\in\Pi_k}\sum_{h=1}^H\left[ \cR_{k,h} (\p_k,\bpi)+\mathit{FP}(\{\pi^\vartheta,p^{\vartheta}_k,s^\vartheta_h,a^\vartheta_h\}_{\vartheta};\lambda)\right],$$
    where $\mathit{FP}(\{\pi^\vartheta,p^{\vartheta}_k,s^\vartheta_h,a^\vartheta_h\}_{\vartheta};\lambda)$ is the fairness penalty and $\lambda$ is the penalty parameter.
    \begin{itemize}
        \item[(a).] \textit{Demographic parity}. For this case, we choose $\mathit{FP}(\{\pi^\vartheta,p^{\vartheta}_k,s^\vartheta_h,a^\vartheta_h\}_{\vartheta};\lambda)$ as $$\lambda(\bP^{\pi^\alpha,p_k^{\alpha}}(a^\alpha_{k,h}=1)-\bP^{\pi^\beta,p_k^{\beta}}(a^\beta_{k,h}=1))^2.$$
        \item[(b).] \textit{Equalized opportunity.} For this case, we choose $\mathit{FP}(\{\pi^\vartheta,p^{\vartheta}_k,s^\vartheta_h,a^\vartheta_h\}_{\vartheta};\lambda)$ as
        \begin{align*}
\lambda\Big(&\bP^{\pi^\alpha,p_k^{\alpha}}(a^\alpha_{k,h}=1|y^\alpha_{k,h}=1)-\bP^{\pi^\beta,p_k^{\beta}}(a^\beta_{k,h}=1|y^\beta_{k,h}=1)\Big)^2
        \end{align*}
    \end{itemize}

    %


\paragraph{Experimental results.}
Figure~\ref{DpRealGenerativeMix_linear_estConstr_policyUpdateAvgReward} shows the Pareto frontier in terms of episodic total return and episodic step-average fairness violation for demographic parity, and Figure \ref{EqOptRealGenerativeMix_linear_estConstr_policyUpdateAvgReward} shows the counterpart for equal opportunity.
Figure~\ref{DpRealGenerativeMix_linear_estConstr_trainEpisodeIndex} and~\ref{DpRealGenerativeMix_linear_policyUpdateAvgReward_trainEpisodeIndex} demonstrate the training dynamics of different algorithms
for demographic parity, and 
Figures~\ref{EqOptRealGenerativeMix_linear_estConstr_trainEpisodeIndex} and~\ref{EqOptRealGenerativeMix_linear_policyUpdateAvgReward_trainEpisodeIndex} demonstrate the counterpart
for equal opportunity.
Our proposed method outperforms the baselines in terms of Pareto frontiers, and converges to a stable level in terms of fairness violation over the training episodes. 
%
%
Overall, from the Pareto frontier, we can observe that our method reaches smaller fairness violation level for non-hispanic white and black individuals in both fairness notions, under a fixed reward level for the decision maker. In addition, from the confidence intervals, we can see that our algorithm has a much narrower confidence band than the other baseline.

\section{Discussion}
In this section, we discuss possible extensions of our framework and future directions.

\subsection{Extension of stepwise fairness notions}\label{subsec:extfairnotion}
In our paper, we mainly consider two popular types of fairness criteria for stepwise fairness constraints, i.e. demographic parity and equalized opportunity. However, our techniques can also be extended in future work to additional types of fairness criteria. In particular, we can consider a family of fairness constraints, as formally introduced in \cite{agarwal2018reductions}. They consider constraints in the form 
$$M\mu(a)\le c$$ for matrix $M$ and vector $c$, where the $j$-th coordinate of $\mu$ is $\mu_j(a)=\bE[g(x,y,a,\vartheta)|E_j]$ for $j\in\mathcal{J}$, and $M\in\bR^{|\mathcal K|\times |\mathcal J|}$, $c\in\bR^{\mathcal K}$. Here, $\mathcal K=\cA\times\cY\times\{+,-\}$ ($+,-$ impose positive/negative sign so as to recover $|\cdot|$ in constraints), for $\cY=\{0,1\}$, and $\mathcal J=(\Lambda\cup \{*\})\times\{0,1\}$. $E_j$ is an event defined with respect to $(x,y,\vartheta)$ and $*$  denotes  the entire probability space. This formulation includes demographic parity and equalized opportunity as special cases.

\subsection{Aggregated fairness notions}\label{subsec:extfairnotion}
Despite the fact that we  consider stepwise fairness constraints, our techniques can also be extended in future work to aggregate fairness notions that consider the entire episodic process. Specifically, we could consider {\em aggregate equalized opportunity} (also called {\em aggregate true positive rate} in \cite{d2020fairness}):
$$\sum_{h=1}^H\bP(a_h=1|y_h=1,\vartheta)\frac{\bP(y_h=1|\vartheta)}{\sum_{h=1}^H\bP(y_h=1|\vartheta)},$$
which can be viewed as a weighted sum of equalized opportunity across steps. This should be relatively straightforward to handle through 
our techniques.

\subsection{Non-episodic, infinite horizon Markov decision processes}
Another natural direction is to extend our framework to non-episodic infinite horizon. Taking DP as an example, we could consider
\begin{align*}
  &\qquad\quad\max_{\bpi\in\Pi_k}\sum_{h=1}^\infty \gamma^h\cR_{h} (\p,\bpi),\\
   &\qquad\qquad~~s.t.~~\forall h\in[H],\\
   &\gamma^h|\bP^{\pi^\alpha,p^{\alpha}}(a^\alpha_{h}=1)-\bP^{\pi^\beta,p^{\beta}}(a^\beta_{h}=1)|\le \hat c_{h}.
\end{align*}
This will involve using a more advanced version of concentration inequality for Markov chain, and we  leave this to future work.

\bibliography{cite,zhun}

\begin{thebibliography}{10}

\bibitem{agarwal2018reductions}
Alekh Agarwal, Alina Beygelzimer, Miroslav Dud{\'\i}k, John Langford, and Hanna
  Wallach.
\newblock A reductions approach to fair classification.
\newblock In {\em International Conference on Machine Learning}, pages 60--69.
  PMLR, 2018.

\bibitem{Ball2019}
Ian Ball.
\newblock Scoring strategic agents.
\newblock In {\em Arxiv}, 2019.

\bibitem{brantley2020constrained}
Kiant{\'e} Brantley, Miroslav Dudik, Thodoris Lykouris, Sobhan Miryoosefi, Max
  Simchowitz, Aleksandrs Slivkins, and Wen Sun.
\newblock Constrained episodic reinforcement learning in concave-convex and
  knapsack settings.
\newblock {\em arXiv preprint arXiv:2006.05051}, 2020.

\bibitem{burhanpurkar2021scaffolding}
Maya Burhanpurkar, Zhun Deng, Cynthia Dwork, and Linjun Zhang.
\newblock Scaffolding sets.
\newblock {\em arXiv preprint arXiv:2111.03135}, 2021.

\bibitem{Calders2009}
Toon Calders, Faisal Kamiran, and Mykola Pechenizkiy.
\newblock Building classifiers with independency constraints.
\newblock In {\em Data mining workshops, ICDMW’09}, 2009.

\bibitem{Chen2020}
Yiling Chen, Yang Liu, and Chara Podimata.
\newblock Learning strategy-aware linear classifiers.
\newblock In {\em Advances in Neural Information Processing Systems 33 (NeurIPS
  2020)}, 2020.

\bibitem{Creager2019}
Elliot Creager, David Madras, Toniann Pitassi, and Richard Zemel.
\newblock Causal modeling for fairness in dynamical systems.
\newblock In {\em Arxiv}, 2019.

\bibitem{d2020fairness}
Alexander D'Amour, Hansa Srinivasan, James Atwood, Pallavi Baljekar, David
  Sculley, and Yoni Halpern.
\newblock Fairness is not static: deeper understanding of long term fairness
  via simulation studies.
\newblock In {\em Proceedings of the 2020 Conference on Fairness,
  Accountability, and Transparency}, pages 525--534, 2020.

\bibitem{deng2020representation}
Zhun Deng, Frances Ding, Cynthia Dwork, Rachel Hong, Giovanni Parmigiani,
  Prasad Patil, and Pragya Sur.
\newblock Representation via representations: Domain generalization via
  adversarially learned invariant representations.
\newblock {\em arXiv}, 2006.11478, 2020.

\bibitem{deng2020interpreting}
Zhun Deng, Cynthia Dwork, Jialiang Wang, and Linjun Zhang.
\newblock Interpreting robust optimization via adversarial influence functions.
\newblock In {\em International Conference on Machine Learning}, pages
  2464--2473. PMLR, 2020.

\bibitem{deng2020towards}
Zhun Deng, Hangfeng He, Jiaoyang Huang, and Weijie Su.
\newblock Towards understanding the dynamics of the first-order adversaries.
\newblock In {\em International Conference on Machine Learning}, pages
  2484--2493. PMLR, 2020.

\bibitem{deng2021toward}
Zhun Deng, Hangfeng He, and Weijie Su.
\newblock Toward better generalization bounds with locally elastic stability.
\newblock In {\em International Conference on Machine Learning}, pages
  2590--2600. PMLR, 2021.

\bibitem{deng2022fifa}
Zhun Deng, Jiayao Zhang, Linjun Zhang, Ting Ye, Yates Coley, Weijie~J Su, and
  James Zou.
\newblock Fifa: Making fairness more generalizable in classifiers trained on
  imbalanced data.
\newblock {\em arXiv preprint arXiv:2206.02792}, 2022.

\bibitem{deng2021improving}
Zhun Deng, Linjun Zhang, Amirata Ghorbani, and James Zou.
\newblock Improving adversarial robustness via unlabeled out-of-domain data.
\newblock In {\em International Conference on Artificial Intelligence and
  Statistics}, pages 2845--2853. PMLR, 2021.

\bibitem{deng2021adversarial}
Zhun Deng, Linjun Zhang, Kailas Vodrahalli, Kenji Kawaguchi, and James~Y Zou.
\newblock Adversarial training helps transfer learning via better
  representations.
\newblock {\em Advances in Neural Information Processing Systems}, 34, 2021.

\bibitem{Dwork2012}
Cynthia Dwork, Moritz Hardt, Toniann Pitassi, OmerReingold, and Richard Zemel.
\newblock Fairness through awarenes.
\newblock In {\em Proceedings of the 3rd innovations in theoretical computer
  science conference}, 2012.

\bibitem{DAmour2020}
Alexander D’Amour, Hansa Srinivasan, James Atwood, Pallavi Baljekar,
  D~Sculley, and Yoni Halpern.
\newblock Fairness is not static: deeper understanding of long term fairness
  via simulation studies.
\newblock In {\em Proceedings of the 2020 Conference on Fairness,
  Accountability, and Transparency}, 2020.

\bibitem{goodfellow2020generative}
Ian Goodfellow, Jean Pouget-Abadie, Mehdi Mirza, Bing Xu, David Warde-Farley,
  Sherjil Ozair, Aaron Courville, and Yoshua Bengio.
\newblock Generative adversarial networks.
\newblock {\em Communications of the ACM}, 63(11):139--144, 2020.

\bibitem{Hardt2015}
Moritz Hardt, Nimrod Megiddo, Christos Papadimitriou, and Mary Wootters.
\newblock Strategic classification.
\newblock In {\em Proceedings of the 2016 ACM Conference on Innovations in
  Theoretical Computer Science}, pages 111--122, 2015.

\bibitem{hardt2016equality}
Moritz Hardt, Eric Price, and Nati Srebro.
\newblock Equality of opportunity in supervised learning.
\newblock {\em Advances in neural information processing systems}, 29, 2016.

\bibitem{Hashimoto2018}
Tatsunori~B Hashimoto, Megha Srivastava, Hongseok Namkoong, and Percy Liang.
\newblock Fairness without demographics in repeated loss minimization.
\newblock In {\em ICML}, 2018.

\bibitem{he2017mask}
Kaiming He, Georgia Gkioxari, Piotr Doll{\'a}r, and Ross Girshick.
\newblock Mask r-cnn.
\newblock In {\em Proceedings of the IEEE international conference on computer
  vision}, pages 2961--2969, 2017.

\bibitem{he2016deep}
Kaiming He, Xiangyu Zhang, Shaoqing Ren, and Jian Sun.
\newblock Deep residual learning for image recognition.
\newblock In {\em Proceedings of the IEEE conference on computer vision and
  pattern recognition}, pages 770--778, 2016.

\bibitem{Jabbari2017}
Shahin Jabbari, Matthew Joseph, Michael Kearns, Jamie Morgenstern, and Aaron
  Roth.
\newblock Fairness in reinforcement learning.
\newblock In {\em Proceedings of the 34th International Conference on Machine
  Learning}, pages 1617--1626, 2017.

\bibitem{Joseph2016}
Matthew Joseph, Michael Kearns, Jamie~H. Morgenstern, and Aaron Roth.
\newblock Fairness in learning: Classic and contextual bandits.
\newblock In {\em Advances in Neural Information Processing Systems 29 (NIPS
  2016)}, 2016.

\bibitem{kawaguchi2022understanding}
Kenji Kawaguchi, Linjun Zhang, and Zhun Deng.
\newblock Understanding dynamics of nonlinear representation learning and its
  application.
\newblock {\em Neural Computation}, 34(4):991--1018, 2022.

\bibitem{kearns2002near}
Michael Kearns and Satinder Singh.
\newblock Near-optimal reinforcement learning in polynomial time.
\newblock {\em Machine learning}, 49(2):209--232, 2002.

\bibitem{kontorovich2008concentration}
Leonid~Aryeh Kontorovich and Kavita Ramanan.
\newblock Concentration inequalities for dependent random variables via the
  martingale method.
\newblock {\em The Annals of Probability}, 36(6):2126--2158, 2008.

\bibitem{Kusner2017}
Matt~J Kusner, Joshua Loftus, Chris Russell, , and Ricardo Silva.
\newblock Counterfactual fairness.
\newblock In {\em Advances in Neural Information Processing Systems}, 2017.

\bibitem{liu2018delayed}
Lydia~T Liu, Sarah Dean, Esther Rolf, Max Simchowitz, and Moritz Hardt.
\newblock Delayed impact of fair machine learning.
\newblock In {\em International Conference on Machine Learning}, pages
  3150--3158. PMLR, 2018.

\bibitem{Mandal2022}
Debmalya Mandal and Jiarui Gan.
\newblock Socially fair reinforcement learning.
\newblock In {\em Arxiv}, 2022.

\bibitem{Shavit2019}
Yonadav Shavit and William~S. Moses.
\newblock Extracting incentives from black-box decisions.
\newblock In {\em Arxiv}, 2019.

\bibitem{sun2020decision}
He~Sun, Zhun Deng, Hui Chen, and David~C Parkes.
\newblock Decision-aware conditional gans for time series data.
\newblock {\em arXiv preprint arXiv:2009.12682}, 2020.

\bibitem{Wen2021}
Min Wen, Osbert Bastani, and Ufuk Topcu.
\newblock Algorithms for fairness in sequential decision making.
\newblock In {\em Proceedings of The 24th International Conference on
  Artificial Intelligence and Statistics}, 2021.

\bibitem{zhang2022and}
Linjun Zhang, Zhun Deng, Kenji Kawaguchi, and James Zou.
\newblock When and how mixup improves calibration.
\newblock In {\em International Conference on Machine Learning}, pages
  26135--26160. PMLR, 2022.

\bibitem{tu2020fair}
Xueru Zhang, Ruibo Tu, Yang Liu, Mingyan Liu, Hedvig Kjellstrom, Kun Zhang, and
  Cheng Zhang.
\newblock How do fair decisions fare in long-term qualification?
\newblock {\em Advances in Neural Information Processing Systems},
  33:18457--18469, 2020.

\bibitem{Zhang2020}
Xueru Zhang, Ruibo Tu, Yang Liu, Mingyan Liu, Hedvig Kjellström, Kun Zhang,
  and Cheng Zhang.
\newblock How do fair decisions fare in long-term qualification?
\newblock In {\em 34th Conference on Neural Information Processing Systems},
  2020.

\end{thebibliography}
\bibliographystyle{plain}

%
%




%

%

\appendix
\section{Solving Our Algorithms via Occupancy Measures}\label{app:occupancy}
Let us use occupation measure to reformulate the problem, and the objective will become occupation measures. For episode $k$, let us denote the $\bP^{\pi^\vartheta_k,p^\vartheta_k}(x^\vartheta_{k,h}=x,y^\vartheta_{k,h}=y,a^\vartheta_{k,h}=a)$ as $\rho^\vartheta_k(x,y,a,h)$, which we call occupancy measures. 
\subsection{Demographic Parity}
For episode $k$, the optimization problem can be reformulated as:
$$\max_\rho \sum_{x,y,a,h,\vartheta}p_\vartheta\rho^\alpha_k(x,y,a,h) \hat{r}^{\vartheta}_k(x,y,a)$$
such that 

$$\forall h~~\Big|\sum_{y,x} \rho^\alpha_k(x,y,a=1,h)-\sum_{y,x} \rho^\beta_k(x,y,a=1,h)\Big|\le \hat {c}_{k,h},$$

$$\forall \vartheta, x, t~~ \frac{\rho^\vartheta_k(x,y=1,a=1,h)}{\sum_{a}\rho^\vartheta_k(x,y=1,a,h)}=\frac{\rho^\vartheta_k(x,y=0,a=1,h)}{\sum_{a}\rho^\vartheta_k(x,y=0,a,h)}\quad\text{this formula makes sure the policy only depends on}~x$$

$$\forall \vartheta, x',y',h\quad \sum_{a} \rho^\vartheta_k(x',y',a,h+1)=\sum_{x,y,a}\rho^\vartheta_k(x,y,a,h)p^\vartheta_{(k)}(x',y'|x,y,a),$$

$$\forall \vartheta,x,y,a,h  ~~0\le \rho^\vartheta_k(x,y,a,h)\le 1, \quad \sum_{x,y,a}\rho^\vartheta_k(x,y,a,h)=1.$$

\subsection{Equal Opportunity}
For episode $k$, the optimization problem can be reformulated as:
$$\max_\rho \sum_{x,y,a,h,\vartheta}p_\vartheta\rho^\alpha_k(x,y,a,h) \hat{r}^{\vartheta}_k(x,y,a)$$
such that 
$$\forall h~~ \left|\frac{\sum_x\rho^\alpha_k(x,y=1,a=1,h)}{\sum_{x,a}\rho^\alpha_k(x,y=1,a,h)}-\frac{\sum_x\rho^\beta_k(x,y=1,a=1,h)}{\sum_{x,a}\rho^\beta_k(x,y=1,a,h)}\right|\le \hat d_{k,h},$$

$$\forall \vartheta, x, t~~ \frac{\rho^\vartheta_k(x,y=1,a=1,h)}{\sum_{a}\rho^\vartheta_k(x,y=1,a,h)}=\frac{\rho^\vartheta_k(x,y=0,a=1,h)}{\sum_{a}\rho^\vartheta_k(x,y=0,a,h)}\quad\text{this formula makes sure the policy only depends on}~x$$

$$\forall \vartheta, x',y',h\quad \sum_{a} \rho^\vartheta_k(x',y',a,h+1)=\sum_{x,y,a}\rho^\vartheta_k(x,y,a,h)p^\vartheta_{(k)}(x',y'|x,y,a),$$

$$\forall \vartheta,x,y,a,h  ~~0\le \rho^\vartheta_k(x,y,a,h)\le 1, \quad \sum_{x,y,a}\rho^\vartheta_k(x,y,a,h)=1.$$


\section{Extension to Multiple Sensitive Attributes}\label{app:multipleatt}
Consider there are multiple attributes $\vartheta\in \Theta=\{\theta_1,\theta_2,\cdots,\theta_q\}$, where $\Theta$ is a set of sensitive attributes.
%
We denote the expected reward for a random individual at time step $h$  as,
$$\cR^*_h(\p^*,\bpi)=\sum_{\vartheta\in\Theta}p_\vartheta\cdot  \bE^{\pi^\vartheta,p^{*\vartheta}}[r^{*\vartheta}(s^\vartheta_h,a^\vartheta_h)].$$

Note here that $\bE^{\pi^\vartheta,p^{*\vartheta}}[\cdot]$ refers to the expected value for an individual in group $\vartheta$, i.e.,
conditioned on the individual being sampled in this group.
Our ideal goal is to obtain the  policy  $\bpi^*$ that solves the following optimization problem:
\begin{align*}
&\qquad\quad\max_{\bpi}\sum_{h=1}^H \cR^*_h (\p^*,\bpi)\\
  &\qquad\qquad \mbox{s.t.}~~\forall h\in[H],  \\
  &\mathit{faircon}(\{\pi^\vartheta,p^{*\vartheta},s^\vartheta_h,a^\vartheta_h\}_{\vartheta\in\Theta}),
\end{align*}
%
where $\mathit{faircon}$ corresponds to a particular fairness concept.

(i) \textit{RL with demographic parity (DP)}. For this case,  $\mathit{faircon}$ is for any $\theta_i,\theta_j\in\Theta$
   $$\quad\bP^{\pi^{\theta_i},p^{*\theta_i}}[a^{\theta_j}_h=1]=\bP^{\pi^\beta,p^{*\theta_j}}[a^{\theta_j}_h=1],$$
   which means at each time step $h$, the decision for individuals from different groups is  statistically independent of the sensitive attribute.
   
(ii) \textit{RL with equalized opportunity (EqOpt)}. For this case,  $\mathit{faircon}$ is for any $\theta_i,\theta_j\in\Theta$
   $$\quad\bP^{\pi^{\theta_i},p^{*\theta_i}}[a^{\theta_j}_h=1|y^{\theta_j}_h=1]=\bP^{\pi^\beta,p^{*\theta_j}}[a^{\theta_j}_h=1|y^{\theta_j}_h=1],$$
which means at each time step $h$, the decision for a random individual from
each of the two different groups is  statistically independent of the sensitive attribute conditioned 
on the individual being qualified.

The corresponding practical optimization for DP and EqOpt can also adapted to multiple sensitive attributes similarly. 
\section{Feasibility Discussions}\label{app:feasibility}

Let us recall our assumption.

\begin{assumption}[Restatement of Assumption~\ref{ass:kernel}]
(a). For all $(s,a,s')\in\cS\times\cA\times\cS$, there exists a universal constant $C>0$, such that $p^{*\vartheta}(s'|s,a)\ge C$ for $\vartheta=\{\alpha,\beta\}$. (b). For all $(x,a)\in\cX\times\cA$, there exists a universal constant $\tilde C$, such that $\pi^{*\vartheta}(a|x)\ge \tilde C$.
\end{assumption}

For the ideal optimization, Assumption~\ref{ass:kernel} make sure every step $\bP^{\pi^{\vartheta},p^{*\vartheta}}(s^\vartheta_h=s,a^\vartheta_h=a)\ge C \min_{x}\pi^\vartheta(a|x)$ for all $s,a$ and $h>1$. Thus, for DP, we can set $\pi^\alpha(a=1|x)=\pi^\beta(a=1|x)=1$ for all $x$, then, we know this policy is a feasible policy. Similar for EqOpt, we also take $\pi^\alpha(a=1|x)=\pi^\beta(a=1|x)=1$ for all $x$, and this is a feasible solution given the event $y^\vartheta_h=1$ is always of positive probability. 

Same argument can be applied to the practical optimization for both DP and EqOpt.

\section{Modelling Individual's Opting In/Out}\label{app:optin}
Recall that in order to analyze the policy effect on the population, we model the behavior of a randomly drawn individual who interacts with the environment across $H$ steps. However, when we gather data from real individuals, they may opt in opt out at a certain step and not interact with the environment throughout the $H$ steps. Nevertheless, we can still aggregate their data and provide estimation of quantities of interest -- a representative randomly drawn individual that interact with the environment for the full episode. Specifically, at episode $k\in[K]$, each time step $h\in [H]$, there are $n_{k,h}$ individuals (people opt in and opt out in the process, so there are different number of people at each time step). For each group $\vartheta$, we pool all those people together and obtain $n^\vartheta_k=\sum_h n^\vartheta_{k,h}$ people in total. We will use the counting method to obtain empirical measures for those quantities, each quantity will need to sum over all the $n^\vartheta_k$ people. For the $i$-th individual, his/her status and action at time $h$ is denoted as $s^{\vartheta,(i)}_{k,h}$ and $a^{\vartheta,(i)}_{k,h}$. Some people will opt out, for example the $i$-th individual opts out at $h+1$, so there is no status for him/her at time $h+1$. Given that, let us define $\mathbbm{1}\{s^{\vartheta,(i)}_{k,h}=s,a^{\vartheta,(i)}_{k,h}=a, s^{\vartheta,(i)}_{k,h+1}=\cdot\}$ to be the indicator, which will be $1$ only when $s^{(i)}_{k,h}=s,a^{(i)}_{k,h}=a$, and that individual still hasn't opted out at time $h+1$ ($s^{\vartheta,(i)}_{k,h+1}$ exists). Similarly, if the the $i$-th individual has opted out at time $h+1$, those indicators used below concerning $s^{\vartheta,(i)}_{k,h+1}$ will be $0$.

For $\vartheta\in\{\alpha,\beta\}$, let $\p_k=(p^\alpha_k,p^\beta_k)$ and $\br_k=(r^\alpha_k,r^\beta_k)$, 
\begin{align*}
N^\vartheta_{k}(s,a)=\max\big\{1, \sum_{t\in [k-1],h\in[H],i\in[n^\vartheta_k]} \mathbbm{1}\{s^{\vartheta,(i)}_{k,h}=s,a^{\vartheta,(i)}_{k,h}=a, s^{\vartheta,(i)}_{k,h+1}=\cdot\}\big\},\\
p^\vartheta_{k} (s'|s,a)=\frac{1}{N^\vartheta_{k}(s,a)}\sum_{t\in [k-1],h\in[H],i\in[n^\vartheta_k]}\mathbbm{1}\{s^{\vartheta,(i)}_{k,h}=s,a^{\vartheta,(i)}_{k,h}=a, s^{\vartheta,(i)}_{k,h+1}=s'\},\\
\hat{r}^\vartheta_{k} (s,a)=\frac{1}{N^\vartheta_{k}(s,a)}\sum_{t\in [k-1],h\in[H],i\in[n^\vartheta_k]}r^\vartheta(s,a)\mathbbm{1}\{s^{\vartheta,(i)}_{k,h}=s,a^{\vartheta,(i)}_{k,h}=a\}.
\end{align*}

Our analysis carries to this setting as long as we assume at each time step $h$ in all episodes, there is an individual that will not opt out in the next time step $h+1$.


\section{Omitted proofs}

 For notation-wise, for simplicity, we \textbf{omit the supscript $\vartheta$} in most of the proofs. In addition, without loss of generality, in the following proofs, \textbf{we assume} $r\in[0,1]$. This is just for proof simplicity, our algorithms can still be applied to the models in the preliminary, which allows negative reward values. Also, for quantities defined below such as $Q^{\pi,p}_r$ and $V^{\pi,p}_r$, we will \textbf{omit subscripts and supscripts} when it is clear from the text. Recall the following concepts:
\paragraph{$Q$ and $V$ functions.} Two of the mostly used concepts in RL are $Q$ and $V$ functions. Specifically, $Q$ functions track the expected reward when a learner starts from state $s\in\cS$. Meanwhile, $V$ functions are the corresponding expected $Q$ functions of the selected action. For a reward function $r$ and a MDP with transition function $p$, $Q$ and $V$ functions are defined as:
\begin{align*}
Q^{\pi,p}_r(s,a,h)&=r(s,a)+\sum_{s'\in\cS}p(s'|s,a)V^{\pi,p}_r(s',h+1),\\
V^{\pi,p}_r(s,h)&=\bE_{a\sim\pi(\cdot|s)}[Q^{\pi,p}_r(s,a,h)],
\end{align*}
where we set $V^{\pi,p}_r(s,H+1)=0$.

\paragraph{Bellman error.}For an arbitrary function $m$, for underlying objectives $m^*$ and $p^*$, the \textit{Bellman error} for $p,m$ under policy $\pi$ at stage $h$ is denoted as:
\begin{equation}
    \cB^{\pi,p}_m(s,a,h)=Q^{\pi,p}_m(s,a,h)-\Big(m^*(s,a)+\sum_{s'}p^*(s'|s,a)V^{\pi,p}_m(s',h+1)\Big).
\end{equation}
\subsection{Proof of lemma \ref{lm:b_k}}

\begin{lemma}[Restatement of lemma \ref{lm:b_k}]
If we take $\hat b^\vartheta_k=\min\Big\{2H,2H\sqrt{\frac{2\ln(16SAHk^2/\delta)}{N^\vartheta_k(s,a)}}\Big\}$, then with probability $1-\delta$, the bonus $\hat b^{\vartheta}_k(s,a)$ is valid for all $k$ episodes and $\vartheta=\{\alpha,\beta\}$. 
\end{lemma}
\begin{proof}
For a fixed $\vartheta\in\{\alpha,\beta\}$, given that $r\in[0,1]$, we have $\sup_{s\in\cS,h\in H} |V(s,h)|\le H$. For a single state-action pair $(s,a)$, by Hoeffding inequality, with probability $1-\delta'$
\begin{equation*}
\Big |\hat r_k(s,a)-r^{*}_k(s,a)+\sum_{s'\in\cS}\big( p_k(s'|s,a)- p^{*}(s'|s,a)\big)V(s',h+1)\Big|\le 2H\sqrt{\frac{2\ln(2/\delta')}{N_k(s,a)}}.
\end{equation*}
Also, by the boundedness of $\hat r_k$ and $V$,
\begin{equation*}
\Big |\hat r_k(s,a)-r^{*}_k(s,a)+\sum_{s'\in\cS}\big( p_k(s'|s,a)- p^{*}(s'|s,a)\big)V(s',h+1)\Big|\le 2H.
\end{equation*}
Further if we take $\delta'=\frac{\delta}{4SAHk^2}$, and further apply union bound on stats and actions, then the failure probability is $\delta/(4k^2)$ for episode $k$. Lastly, bounding across episodes, we have the failure probabilty is:
$$\sum_{i=1}^K\frac{\delta}{4k^2}\le \delta.$$ 

Since there are two values for $\vartheta$, again we apply union bound, then we have the final result.
\end{proof}
\subsection{Proof of Lemma \ref{lm:c_kd_k}}
To prove our result, we need the following handy lemma.
\begin{lemma}[Simulation lemma \cite{kearns2002near}]\label{lm:simulation}
For any policy $\pi$, objective $m$, transition probabilty $p$, and  underlying objectives $m^*,p^*$, it holds that
$$\bE^{\pi}V^{\pi,p}_m(s_1,1)-\bE^{\pi}V^{\pi,p^*}_{m^*}(s_1,1)=\bE^{\pi}\Big[\sum_{h=1}^H\cB^{\pi,p}_m(s_h,a_h,h)\Big].$$
\end{lemma}

\begin{lemma}[Restatement of Lemma \ref{lm:c_kd_k}]
Denote $N^{\vartheta,min}_{k}=\min_{s,a}N^\vartheta_k(s,a)$, for any $\{\epsilon_k\}_{k=1}^K$, with probability at least $1-\delta$, we take
$$\hat c_{k,h}=\sum_{\vartheta\in\{\alpha,\beta\}}H\sqrt{\frac{2S\ln(16SAHk^2/(\epsilon_k\delta))}{N^{\vartheta,min}_{k}}}+2\epsilon_k HS,$$
then $\hat c_{k,h}$ is compatible for all $h\in[H]$ and $k\in[K]$.
\end{lemma}
\begin{proof}
For a specific time step $h^*$, let us consider taking

\begin{equation}
  m_{h^*}(s_{k,h},a_{k,h})=m^*_{h^*}(s_{k,h},a_{k,h})=\begin{cases}
    \mathbbm{1}\{a_{k,h}=1\}, & \text{if $h=h^*$}.\\
    0, & \text{otherwise}.
  \end{cases}
\end{equation}
And it is easy to observe that
$$\bE^{\pi}V^{\pi,p}_{m_{h^*}}(s_1,1)=\bP^{\pi,p}(a_{k,h^*}=1).$$

Thus, in order to bound
$$|\bP^{\pi^*,p_k}(a_{k,h^*}=1)-\bP^{\pi^*,p^*}(a_{k,h^*}=1)|,$$
it is equivalent is to bound
$$|\bE^{\pi^*}V^{\pi^*,p_k}_{m_{h^*}}(s_1,1)-\bE^{\pi^*}V^{\pi^*,p^*}_{m_{h^*}}(s_1,1)|=\Big|\bE^{\pi^*}\Big[\sum_{h=1}^H\cB^{\pi^*,p_k}_{m_{h^*}}(s_h,a_h,h)\Big]\Big|.$$
Here, a slight fine-grained analysis suggests that we can replace $\sum_{h=1}^H$ to $\sum_{h=1}^{h^*}$, but this change cannot change the order of the final bound, so for simplicity, we still use $\sum_{h=1}^H$.

Now, let us analyze $\Big|\cB^{\pi^*,p_k}_{m_{h^*}}(s_h,a_h,h)\Big|.$ Specifically, 
\begin{align*}
\Big|\cB^{\pi^*,p_k}_{m_{h^*}}(s_h,a_h,h)\Big|&=\Big|\sum_{s'\in \cS}\Big(p_k(s'|s,a)-p^*(s'|s,a)\Big)V^{\pi^*,p_k}_{m_{h^*}}(s')\Big|.
\end{align*}

Since $V^{\pi^*,p_k}_{m_{h^*}}(s')$ (we will use $V$ for simplicity from now on) is data dependent, we need to to use a union bound argument. Notice $V(s)\in[0,1]$ for all $s$, thus, we can let $\Psi$ to be a $\epsilon$-net on $[0,1]^S$. For any fixed $V\in \Psi$, by similar proof as in Lemma~\ref{lm:b_k}, we have with probability at least $1-\delta'$, for all $k\in[K]$,
$$\Big|\sum_{s'\in \cS}\Big(p_k(s'|s,a)-p^*(s'|s,a)\Big)V(s')\Big|\le\sqrt{\frac{2\ln(8SAk^2/\delta')}{N_k(s,a)}}.$$

The cardinality of $\Psi$ is $(1/\epsilon)^S$. Thus, by using union bound and let $\delta=\delta'/(1/\epsilon)^S$ (we don't need to union bound over $H$ because we have bounded for all elements in epsilon nets), we have with probability at least $1-\delta$, for all $h\in[H]$,
$$\Big|\cB^{\pi^*,p_k}_{m_{h^*}}(s_h,a_h,h)\Big|\le\sqrt{\frac{2S\ln(8SAHk^2/(\epsilon\delta))}{N_k(s_h,a_h)}}+\epsilon S.$$

Notice our argument still valid if we set $\epsilon$ as $\epsilon_k$ for episode $k$. Then, we have that with probability at least $1-\delta$ 

$$\hat c_{k,h}=\sum_{\vartheta\in\{\alpha,\beta\}}2\sqrt{\frac{2S\ln(16SAk^2/(\epsilon_k\delta))}{\min_{s,a}N^\vartheta_k(s,a)}}+2\epsilon_k HS.$$
is compatible for all $k\in[K]$.

\end{proof}
\subsection{Proof of Lemma \ref{lm:d}}
\begin{lemma}[Restatement of Lemma~\ref{lm:d}]
Denote $p^{\vartheta,min}_{k}=\min_{s,a}p_k(y=1|s,a)$, For any $\{\epsilon_k\}_{k=1}^K$, with probability at least $1-\delta$, we take
$$\hat d_{k,h}=\sum_{\vartheta\in\{\alpha,\beta\}}\frac{3H\sqrt{\frac{2S\ln(32SAk^2/(\epsilon_k\delta))}{N^{\vartheta,min}_{k}}}+3\epsilon_k HS}{p^{\vartheta,min}_{k}\left(p^{\vartheta,min}_{k}-\sqrt{\frac{4\ln 2+2\ln(4SAk^2/\delta)}{N^{\vartheta,min}_{k}}}\right)}$$
if $p^{\vartheta,min}_{k}>\sqrt{\frac{4\ln 2+2\ln(4SAk^2/\delta)}{N^{\vartheta,min}_{k}}}$; Otherwise, we set $\hat d_{k,h}=1$.
Then $\hat d_{k,h}$ is compatible for all $h\in[H]$ and $k\in[K]$.
\end{lemma}

\begin{proof} 
By Bretagnolle-Huber-Carol's inequality, we know that with probability at least $1-\delta$, for all $(s,a)\in\cS\times\cA$ and $k\in[K]$
$$\sum_{y\in\cY}|p_k(y|s,a)-p^*(y|s,a)|\le \sqrt{\frac{4\ln 2+2\ln(SAk^2/\delta)}{N_k(s,a)}}.$$

By proof as in Lemma~\ref{lm:c_kd_k}, we know that by taking 
\begin{equation}
  m_{h^*}(s_{k,h},a_{k,h})=m^*_{h^*}(s_{k,h},a_{k,h})=\begin{cases}
    \mathbbm{1}\{a_{k,h}=1,y_{k,h}=1\}, & \text{if $h=h^*$},\\
    0, & \text{otherwise},
  \end{cases}
\end{equation}

we have for any $\epsilon>0$, for any $a\in\cA$
$$|\bP^{\pi^*,p_k}(a_{k,h^*}=a,y_{k,h^*}=1)-\bP^{\pi^*,p^*}(a_{k,h^*}=a,y_{k,h^*}=1)|\le H \sqrt{\frac{2S\ln(8SAk^2/(\epsilon\delta))}{\min_{s,a}N_k(s,a)}}+\epsilon HS,$$
with probability at least $1-\delta$.


On the other hand, by summing over $a$ for $\bP^{\pi^*,p_k}(a_{k,h^*}=a,y_{k,h^*}=1)-\bP^{\pi^*,p^*}(a_{k,h^*}=a,y_{k,h^*}=1)$, we have 

\begin{align*}
|\bP^{\pi^*,p_k}(y_{k,h^*}=1)-\bP^{\pi^*,p^*}(y_{k,h^*}=1)|&\le \sum_a|\bP^{\pi^*,p_k}(a_{k,h^*}=a,y_{k,h^*}=1)-\bP^{\pi^*,p^*}(a_{k,h^*}=a,y_{k,h^*}=1)|\\
&\le 2H\sqrt{\frac{2S\ln(8SAk^2/(\epsilon\delta))}{\min_{s,a}N_k(s,a)}}+2\epsilon HS,
\end{align*}

If $h>1$
\begin{align*}
\bP^{\pi^*,p^*}(y_{k,h}=1)&\ge\sum_{s,a} p^*(y_{k,h}=1|s_{k,h-1}=s,a_{k,h-1}=a)\bP^{\pi^*,p^*}(s_{k,h-1}=s,a_{k,h-1}=a)\\
&\ge \min_{s,a}p^*(y=1|s,a).
\end{align*}
Since $s_1$ can be chosen by us, so we can make sure $\bP^{\pi^*,p^*}(y_{k,1}=1)$ bounded away from $0$. Actually, even we set $y_{k,1}=0$, conditioning on $\emptyset$ automatically satisfy EqOpt constraint.

Similarly, 
\begin{align*}
\bP^{\pi^*,p_k}(y_{k,h}=1)&\ge\sum_{s,a} p_k(y_{k,h}=1|s_{k,h-1}=s,a_{k,h-1}=a)\bP^{\pi^*,p_k}(s_{k,h-1}=s,a_{k,h-1}=a)\\
&\ge \min_{s,a}p_k(y=1|s,a).
\end{align*}
 By Bretagnolle-Huber-Carol's inequality, with probability at least $1-\delta$, for all $(s,a)\in\cS\times \cA$
$$|p_k(y=1|s,a)-p^*(y=1|s,a)|\le \sqrt{\frac{4\ln 2+2\ln(SAk^2/\delta)}{N_k(s,a)}}.$$

As a result, if $$\min_{s,a} p_k(y=1|s,a)>2\sqrt{\frac{4\ln 2+2\ln(SAk^2/\delta)}{\min_{s,a}N_k(s,a)}},$$
then, with probability at least $1-\delta$,
 $$\min_{s,a} p^*(y=1|s,a)>\min_{s,a}p_k(y=1|s,a)-\sqrt{\frac{4\ln 2+2\ln(SAk^2/\delta)}{\min_{s,a}N_k(s,a)}},$$

which further leads to
\begin{align*}
&|\bP^{\pi^*,p_k}(a_{k,h^*}=1|y_{k,h^*}=1)-\bP^{\pi^*,p^*}(a_{k,h^*}=1|y_{k,h^*}=1)|\\
&\le\frac{|\bP^{\pi^*,p_k}(a_{k,h^*}=1,y_{k,h^*}=1)-\bP^{\pi^*,p^*}(a_{k,h^*}=1,y_{k,h^*}=1)|+|\bP^{\pi^*,p_k}(a_{k,h^*}=1)-\bP^{\pi^*,p^*}(a_{k,h^*}=1)|}{p_k(y=1|s,a)p^*(y=1|s,a)}\\
&\le\frac{|\bP^{\pi^*,p_k}(a_{k,h^*}=1,y_{k,h^*}=1)-\bP^{\pi^*,p^*}(a_{k,h^*}=1,y_{k,h^*}=1)|+|\bP^{\pi^*,p_k}(a_{k,h^*}=1)-\bP^{\pi^*,p^*}(a_{k,h^*}=1)|}{\min_{s,a}p_k(y=1|s,a)(\min_{s,a}p_k(y=1|s,a)-\sqrt{\frac{4\ln 2+2\ln(SAk^2/\delta)}{\min_{s,a}N_k(s,a)}})}\\
&\le \frac{3H\sqrt{\frac{2S\ln(8SAk^2/(\epsilon\delta))}{\min_{s,a}N_k(s,a)}}+3\epsilon HS}{\min_{s,a}p_k(y=1|s,a)(\min_{s,a}p_k(y=1|s,a)-\sqrt{\frac{4\ln 2+2\ln(SAk^2/\delta)}{\min_{s,a}N_k(s,a)}})}.
\end{align*}

Notice our argument still valid if we set $\epsilon_k$ for episode $k$ and apply union bounds for all the events mentioned above and $\vartheta=\{\alpha,\beta\}$, then, we have that with probability at least $1-\delta$

$$\hat d_{k,h}=\left\{\begin{matrix}
	\sum_{\vartheta}\frac{3H\sqrt{\frac{2S\ln(32SAk^2/(\epsilon_k\delta))}{N^{\vartheta,min}_{k}}}+3\epsilon_k HS}{p^{\vartheta,min}_{k}\left(p^{\vartheta,min}_{k}-\sqrt{\frac{4\ln 2+2\ln(4SAk^2/\delta)}{N^{\vartheta,min}_{k}}}\right)} ,&\text{if}~p^{\vartheta,min}_{k}>\sqrt{\frac{4\ln 2+2\ln(4SAk^2/\delta)}{N^{\vartheta,min}_{k}}};\\ 
	1,&\text{otherwise};
\end{matrix}\right.$$
is compatible for all $k\in[K]$.
\end{proof}

\subsection{Proof of Theorem~\ref{thm:reward}}
We first need a lemma regarding the lower bound of $\min_{s,a}N^\vartheta_k(s,a)$.

\paragraph{Restatement of Result in \cite{kontorovich2008concentration}}In order to do that, let us consider a simplified variant of Theorem 1.1 in \cite{kontorovich2008concentration}. Let $Z_i\in S$, where $S$ is a finite set, and $Z=(Z_1,Z_2,\cdots,Z_n)$. We further denote $Z^j_i=(Z_i,Z_{i+1},\cdots,Z_j)$ as a random vector for  $1\leq i< j\leq n$.  Correspondingly, we let  $z^j_i=(z_i,z_{i+1},\cdots,z_j)$ be a subsequence for $(z_1,z_2,\cdots,z_n)$. And let 
$$\bar{\eta}_{i,j}=\sup_{v^{i-1}_1\in S^{i-1}, w,w'\in S,~\bP(Z^i_1=Y^{i-1}w)>0,~\bP(Z^i_1=V^{i-1}w')>0}\eta_{i,j}(v^{i-1}_1,w,w'),$$
where 
$$\eta_{i,j}(v^{i-1}_1,w,w')=TV\Big(\cD(Z^n_j|Z^i_1=v^{i-1}_1w),\cD(Z^n_j|Z^i_1=v^{i-1}_1w')\Big).$$
Here $TV$ is the total variational distance, and $\cD(Z^n_j|Z^i_1=v^i_1w)$ is the conditional distribution of $Z^n_j$ conditioning on $\{Z^i_1=v^i_1w\}$.

Let $H_n$ be $n\times n$ upper triangular matrix defined by
$$(H_n)_{ij}=\left\{\begin{matrix}
 1&i=j   \\ 
 \bar{\eta}_{i,j}&i<j   \\ 
 0&   o.w.
\end{matrix}\right.$$
Then, 
$$\|H_n\|_\infty=\max_{1\leq i\leq n}J_{n,i},$$
where 
$$J_{n,i}=1+\bar{\eta}_{i,i+1}+\cdots+\bar{\eta}_{i,n},$$
and $J_{n,n}=1$.

\begin{theorem}[Variant of Result in \cite{kontorovich2008concentration}]\label{thm:konto}
Let $f$ be a $L_f$-Lipschitz function (with respect to the Hamming distance) on $S^n$ for some constant $L_f>0$. Then, for any $t>0$,
$$\bP(|f(Z)-Ef(z)|\geq t)\leq 2\exp\Big(-\frac{t^2}{2nL^2_f\|H_n\|^2_\infty}\Big).$$

\end{theorem}

\begin{lemma}
Under Assumption~\ref{ass:kernel}, with probability at least $1-\delta$, for all $k\in[K]$, $(s,a)\in\cS\times\cA$, $\vartheta\in\{\alpha,\beta\}$
$$N^\vartheta_k(s,a)\ge \left(1-\sqrt{\frac{2H^2\ln(4k^2SA/\delta)}{c\eta_kn_k(H-1)(k-1)}}\right)c\eta_kn_k(H-1)(k-1).$$
\end{lemma}

\begin{proof}

By $(a)$ in Assumption~\ref{ass:kernel}, we know that for any $(s,a)$, $(s',a')$
$$\bP^{\pi^{\vartheta}_k,p^{*\vartheta}}(s^\vartheta_{k,h}=(x,y),a^\vartheta_{k,h}=a)\ge\pi^\vartheta_k(a^\vartheta_{k,h}=a|x^\vartheta_{k,h}=x)p^{*\vartheta}(s^{\vartheta}_{k,h}=(x,y)|s^{\vartheta}_{k,h-1}=s',a^\vartheta_{k,h-1}=a')\ge C\eta_k.$$

As $\eta_k$ is decreasing with respect to $k$, we must have for all $t\le k$,
$$\bP^{\pi^{\vartheta}_t,p^{*\vartheta}}(s^\vartheta_{t,h},a^\vartheta_{t,h})\ge\pi^\vartheta_t(a^\vartheta_{t,h}|x^\vartheta_{t,h})p^{*\vartheta}(s^{\vartheta}_{t,h}|s^{\vartheta}_{t,h-1},a^\vartheta_{t,h-1})\ge C\eta_k.$$

Notice each individual is independent thus,

$$\bE\frac{N_k(s,a)}{(H-1)(K-1)}\ge C\eta_kn_k.$$

Meanwhile, the data between different episodes are independent, thus, using the notation in \cite{kontorovich2008concentration}, we know $\bar{\eta}_{ij}=0$ if $|i-j|\ge H$. Notice total variational distance is always bounded by $1$, so
$$\|H_n\|_\infty=\max_{1\leq i\leq n}J_{n,i}\leq H.$$

then, by applying Theorrem~\ref{thm:konto} about the concentration result for non-stationary Markov chain, 
$$\bP\left(|N_t(s,a)-\bE N_t(s,a)|\ge \varepsilon cn_k\eta_k(H-1)(k-1)\right)\le 2\exp\left(\frac{-\varepsilon^2C\eta_kn_k(H-1)(k-1)}{2H^2}\right).$$

In other words, with probability at least $1-\delta'$,
$$N_k(s,a)\ge \left(1-\sqrt{\frac{2H^2\ln(2/\delta')}{C\eta_kn_k(H-1)(k-1)}}\right)C\eta_kn_k(H-1)(k-1)$$

Taking $\delta'=\delta/(2SAk^2)$, we have with probability at least $1-\delta$, for all $k\in[K]$, $(s,a)\in\cS\times\cA$, $\vartheta\in\{\alpha,\beta\}$
$$N^\vartheta_k(s,a)\ge \left(1-\sqrt{\frac{2H^2\ln(4k^2SA/\delta)}{C\eta_kn_k(H-1)(k-1)}}\right)C\eta_kn_k(H-1)(k-1).$$

\end{proof}

Next, we will use the concept of \textit{optimism} in \cite{brantley2020constrained} in our proof.
\begin{definition}[Optimism]
We call $(p_k,r_k)$ is optimisitc if 
$$\bE\Big[V^{\pi^*,p_k}_{r_k}(s_1,1)\Big]\ge \bE\Big[V^{\pi^*,p^*}_{r^*}(s_1,1)\Big].$$
\end{definition}

\begin{lemma}
If $\hat b_k$ is valid in episode $k$ for all $k$ simultaneously, we have 
$$\bE\Big[V^{\pi^*,p_k}_{r_k}(s_1,1)\Big]\ge \bE\Big[V^{\pi^*,p^*}_{r^*}(s_1,1)\Big].$$
\end{lemma}
\begin{proof}
This proof mainly follows \cite{brantley2020constrained} by using induction.

Since the setting ends at episode $H$,
$$Q^{\pi^*,p_k}_{r_k}(s,a,H+1)=Q^{\pi^*,p^*}_{r^*}(s,a,H+1)=0.$$

We assume that the inductive hypothesis $Q^{\pi^*,p_k}_{r_k}(s,a,h)\ge Q^{\pi^*,p^*}_{r^*}(s,a,h+1)$ (thus, $V^{\pi^*,p_k}_{r_k}(s,h+1)\ge V^{\pi^*,p^*}_{r^*}(s,h+1)$ holds)). Then, for $h$,
\begin{align*}
Q^{\pi^*,p_k}_{r_k}(s,a,h+1)&=r_k+\sum_{s'\in\cS}p_k(s'|s,a)V^{\pi^*,p_k}_{r_k}(s',h+1)\\
&\ge r_k+\sum_{s'\in\cS}p_k(s'|s,a)V^{\pi^*,p^*}_{r^*}(s',h+1).
\end{align*}

Meanwhile, we know,
$$Q^{\pi^*,p^*}_{r^*}(s,a,h)=r^*(s,a)+ \sum_{s'\in\cS}p^*(s'|s,a)V^{\pi^*,p^*}_{r^*}(s',h+1).$$

Subtracting the above two formulas, we have 
\begin{align*}
Q^{\pi^*,p_k}_{r_k}(s,a,h)-Q^{\pi^*,p^*}_{r^*}(s,a,h)&\ge (\hat{r}_k(s,a)+\hat{b}_k(s,a)-r^*(s,a))\\
&+ \sum_{s'\in\cS}(p_k(s'|s,a)-p^*(s'|s,a))V^{\pi^*,p^*}_{r^*}(s',h+1)\ge 0
\end{align*}
where the last inequality holds because of the bonuses are valid.
\end{proof}

Let us summarize briefly and informally: with probability $1-4\delta$, we have:

\begin{itemize}
    \item[1.] $\{c_{k,h}\}_{k,h}$ are compatible;
    \item[2.] $\{d_{k,h}\}_{k,h}$ are compatible;
    \item[3.] $\{b^\vartheta_{k,h}\}_{k,h}$ are valid;
    \item[4.] $N^\vartheta_k(s,a)\ge \left(1-\sqrt{\frac{2H^2\ln(4k^2SA/\delta)}{c\eta_kn_k(H-1)(k-1)}}\right)C\eta_kn_k(H-1)(k-1).$
\end{itemize}

We denote event $\cE_1$ as the events 1, 3, 4 hold simultaneously. We denote event $\cE_2$ as the events 2, 3, 4 hold simultaneously. 

\begin{lemma}\label{lm:keyreg}
When for DP case, $\cE_1$ holds (similar for EqOpt case, $\cE_2$ holds), and when $\eta_k\le \tilde C$, we have 
\end{lemma}

\begin{proof}

Recall
$$\cR^*_h(\p^*,\bpi)=\sum_{\vartheta\in\{\alpha,\beta\}}p_\vartheta\cdot  \bE^{\pi^\vartheta,p^{*\vartheta}}[r^{*\vartheta}(s^\vartheta_h,a^\vartheta_h)].$$
Our aim is to bound
$$\cR_{reg}(k)= \frac{1}{H}\sum_{h=1}^H\Big(\cR^*_h (\p^*,\bpi^{*})-\cR^*_h (\p^*,\bpi_k)\Big)=\frac{1}{H}\sum_{\vartheta}p^\vartheta\left(\bE\Big[V^{\pi^{*\vartheta},p^{*\vartheta}}_{r^{*\vartheta}}(s_1,1)\Big]-\bE\Big[V^{\pi^{\vartheta}_k,p^{*\vartheta}}_{r^{*\vartheta}}(s_1,1)\Big]\right).$$

Let us first study $\sum_{\vartheta}p^\vartheta\left(\bE\Big[V^{\pi^{*\vartheta},p^{*\vartheta}}_{r^{*\vartheta}}(s_1,1)\Big]-\bE\Big[V^{\pi^{\vartheta}_k,p^{*\vartheta}}_{r^{*\vartheta}}(s_1,1)\Big]\right)$.

If $\{b^\vartheta_{k,h}\}_{k,h}$ are valid, by optimism, we have 
$$\bE\Big[V^{\pi^{*\vartheta},p^{*\vartheta}}_{r^{*\vartheta}}(s_1,1)\Big]\le \bE\Big[V^{\pi^{*\vartheta},p^\vartheta_k}_{r^\vartheta_k}(s_1,1)\Big].$$

As a result, we have 
\begin{align*}
\sum_{\vartheta}p^\vartheta\left(\bE\Big[V^{\pi^{*\vartheta},p^{*\vartheta}}_{r^{*\vartheta}}(s_1,1)\Big]-\bE\Big[V^{\pi^{\vartheta}_k,p^{*\vartheta}}_{r^{*\vartheta}}(s_1,1)\Big]\right)&\le \sum_{\vartheta}p^\vartheta\left(\bE\Big[V^{\pi^{*\vartheta},p^\vartheta_k}_{r^\vartheta_k}(s_1,1)\Big]-\bE\Big[V^{\pi^{\vartheta}_k,p^{*\vartheta}}_{r^{*\vartheta}}(s_1,1)\Big]\right).
\end{align*}

Throughout the proof and proofs afterwards, let us take $\eta_k=k^{-1/3}$.

If $\eta_k\le \tilde C$ (equivalently $k>(\tilde C)^{-3}$ if we take $\eta_k=k^{-1/3}$), then by compatibility of $\{\hat{c}_{k,h}\}_{k,h}$ or $\{\hat{d}_{k,h}\}_{k,h}$, $(\pi^{*\alpha},\pi^{*\beta})$ is a feasible solution to our algorithm, as a result

\begin{align*}
\sum_{\vartheta}p^\vartheta\left(\bE\Big[V^{\pi^{*\vartheta},p^{*\vartheta}}_{r^{*\vartheta}}(s_1,1)\Big]-\bE\Big[V^{\pi^{\vartheta}_k,p^{*\vartheta}}_{r^{*\vartheta}}(s_1,1)\Big]\right)&\le \sum_{\vartheta}p^\vartheta\left(\bE\Big[V^{\pi^{*\vartheta},p^\vartheta_k}_{r^\vartheta_k}(s_1,1)\Big]-\bE\Big[V^{\pi^{\vartheta}_k,p^{*\vartheta}}_{r^{*\vartheta}}(s_1,1)\Big]\right)\\
&\le \sum_{\vartheta}p^\vartheta\left(\bE\Big[V^{\pi^{\vartheta}_k,p^\vartheta_k}_{r^\vartheta_k}(s_1,1)\Big]-\bE\Big[V^{\pi^{\vartheta}_k,p^{*\vartheta}}_{r^{*\vartheta}}(s_1,1)\Big]\right)\\
&\le \sum_{\vartheta}p^\vartheta\left[\cB^{\pi^{\vartheta}_k,p^\vartheta_k}_{r^\vartheta_k}(s_{k,h},a_{k,h},h)\right].
\end{align*}
Thus,
\begin{align*}
\cR_{reg}(k)&= \frac{1}{H}\sum_{h=1}^H\Big(\cR^*_h (\p^*,\bpi^{*})-\cR^*_h (\p^*,\bpi_k)\Big)\\
&=\frac{1}{H}\sum_{\vartheta}p^\vartheta\left(\bE\Big[V^{\pi^{*\vartheta},p^{*\vartheta}}_{r^{*\vartheta}}(s_1,1)\Big]-\bE\Big[V^{\pi^{\vartheta}_k,p^{*\vartheta}}_{r^{*\vartheta}}(s_1,1)\Big]\right)\\
&\le \frac{1}{H}\sum_{\vartheta}p^\vartheta\left[\cB^{\pi^{\vartheta}_k,p^\vartheta_k}_{r^\vartheta_k}(s_{k,h},a_{k,h},h)\right].
\end{align*}

By Lemma B.4 of \cite{brantley2020constrained}, with probability $1-2\delta$ for any $\vartheta\in\{\alpha,\beta\}$,
$$\left|\cB^{\pi^{\vartheta}_k,p^\vartheta_k}_{r^\vartheta_k}(s_{k,h},a_{k,h},h)\right|\le 4H^2\sqrt{\frac{2S\ln(16S^2AH^2k^3/\delta)}{\min_{s,a}N_k(s,a)}}+\frac{1}{k}.$$

Then, when $k\ge \frac{8H^2\ln(4k^2SA/\delta)}{C\eta_kn_k(H-1)}$, that is  $k\ge C'+ \left(\frac{8H^2\ln(4SA/\delta)}{Cn_k(H-1)}\right)^{3}$ for some constant $C'$, we have

$$N^\vartheta_k(s,a)\ge \frac{1}{2}C\eta_kn_k(H-1)(k-1). $$

Thus,
$$\left|\cB^{\pi^{\vartheta}_k,p^\vartheta_k}_{r^\vartheta_k}(s_{k,h},a_{k,h},h)\right|\le 4H^2\sqrt{\frac{4S\ln(16S^2AH^2k^3/\delta)}{Ck^{-1/3}n_k(H-1)(k-1)}}+\frac{1}{k}.$$
Thus, with probability $1-2\delta$
$$\cR_{reg}(k)\le 4H\sqrt{\frac{4S\ln(16S^2AH^2k^3/\delta)}{Ck^{-1/3}n_k(H-1)(k-1)}}+\frac{1}{kH}$$

\end{proof}

\begin{theorem}[Restatement of Theorem~\ref{thm:reward}]
For $type\in\{\mathit{DP},\mathit{EqOpt}\}$, with probability at least $1-\delta$, there exists a threshold $T=\cO\left(\Big(\frac{H\ln(SA/\delta)}{n_k}\Big)^{3}\right)$, such that for all $k\ge T$,
$$\cR^{\text{type}}_{\mathrm{reg}}(k)=\cO\left(Hk^{-\frac{1}{3}}\sqrt{HS\ln(S^2AH^2k^3/\delta)}\right).$$
\end{theorem}
\begin{proof}
Noticing $n_k\ge 1$ and the result follows immediately from Lemma~\ref{lm:keyreg}.
\end{proof}

\subsection{Proof of Theorem~\ref{thm:constraint}}

Recall for the fairness constraints, we consider violation for each type of constraint in episode $k$ as the following:
$$\cC_{reg}^{DP}(k)=\frac{1}{H}\sum_{h=1}^H \big|\bP^{\pi^{\alpha}_k,p^{*\alpha}}(a^\alpha_{k,h}=1)-\bP^{\pi^\beta_k,p^{*\beta}}(a^\beta_{k,h}=1)\big|.$$
and 
\begin{align*}
 \cC_{reg}^{EqOpt}(k)=\frac{1}{H}\sum_{h=1}^H \big|\bP^{\pi^{\alpha}_k,p^{*\alpha}}(a^\alpha_{k,h}=1\big|y^\alpha_{k,h}=1)-\bP^{\pi^\beta_k,p^{*\beta}}(a^\beta_{k,h}=1|y^\beta_{k,h}=1)\big|.  
\end{align*}
\begin{theorem}[Restatement of Theorem~\ref{thm:constraint}]
For $type\in\{\mathit{DP},\mathit{EqOpt}\}$, with probability at least $1-\delta$, there exists a threshold $T=\cO\left(\Big(\frac{H\ln(SA/\delta)}{n_k}\Big)^{3}\right)$, such that for all $k\ge T$,
$$\cC_{\mathrm{reg}}^{\text{type}}(k)\le \cO\left( k^{-\frac{1}{3}}\sqrt{SH\ln(S^2HAk^3/\delta)}\right).$$

\end{theorem}

\begin{proof}

Let us first consider $\cC_{reg}^{DP}(k)$. Notice

\begin{align*}
\cC_{reg}^{DP}(k)&=\frac{1}{H}\sum_{h=1}^H \big|\bP^{\pi^{\alpha}_k,p^{*\alpha}}(a^\alpha_{k,h}=1)-\bP^{\pi^\beta_k,p^{*\beta}}(a^\beta_{k,h}=1)\big|\\
&\le \frac{1}{H}\sum_{h=1}^H \big|\bP^{\pi^{\alpha}_k,p^{\alpha}_k}(a^\alpha_{k,h}=1)-\bP^{\pi^\beta_k,p^{\beta}_k}(a^\beta_{k,h}=1)\big|+\frac{1}{H}\sum_{h=1}^H \big|\bP^{\pi^{\alpha}_k,p^{*\alpha}}(a^\alpha_{k,h}=1)-\bP^{\pi^\alpha_k,p^{\alpha}_k}(a^\alpha_{k,h}=1)\big|\\
&+\frac{1}{H}\sum_{h=1}^H \big|\bP^{\pi^{\beta}_k,p^{*\beta}}(a^\beta_{k,h}=1)-\bP^{\pi^\beta_k,p^{\beta}_k}(a^\beta_{k,h}=1)\big|\\
&\le \sum_h\frac{\hat{c}_{k,h}}{H}+\frac{1}{H}\sum_{h=1}^H \big|\bP^{\pi^{\alpha}_k,p^{*\alpha}}(a^\alpha_{k,h}=1)-\bP^{\pi^\alpha_k,p^{\alpha}_k}(a^\alpha_{k,h}=1)\big|\\
&+\frac{1}{H}\sum_{h=1}^H \big|\bP^{\pi^{\beta}_k,p^{*\beta}}(a^\beta_{k,h}=1)-\bP^{\pi^\beta_k,p^{\beta}_k}(a^\beta_{k,h}=1)\big|.
\end{align*}

Notice that switching $\pi^{*\vartheta}$ to $\pi^\vartheta_k$ doesn't change our argument in Lemma~\ref{lm:c_kd_k}, thus with probability at least $1-\delta$, for $\epsilon_k$ (we take $\epsilon_k$ to be consistent to $\hat{c}_{k,h}$) for both $\vartheta=\{\alpha,\beta\}$

$$\big|\bP^{\pi^{\vartheta}_k,p^{*\vartheta}}(a^\vartheta_{k,h}=1)-\bP^{\pi^\vartheta_k,p^{\vartheta}_k}(a^\vartheta_{k,h}=1)\big|\le H\sqrt{\frac{2S\ln(16SAk^2/(\epsilon_k\delta))}{\min_{s,a}N^\vartheta_k(s,a)}}+2\epsilon_k HS.$$

Thus, with probability at least $1-2\delta$,

$$\cC_{reg}^{DP}(k)\le \sum_\vartheta 2H\sqrt{\frac{2S\ln(16SAk^2/(\epsilon_k\delta))}{\min_{s,a}N^\vartheta_k(s,a)}}+4\epsilon_k HS.$$

By taking $\epsilon_k=1/(kHS)$,
$$\cC_{reg}^{DP}(k)\le \sum_\vartheta 2H\sqrt{\frac{2S\ln(16S^2Ak^3H/(\delta))}{\min_{s,a}N^\vartheta_k(s,a)}}+\frac{4}{k}.$$

Now, let us consider $$\cC_{reg}^{EqOpt}(k)=\frac{1}{H}\sum_{h=1}^H \big|\bP^{\pi^{\alpha}_k,p^{*\alpha}}(a^\alpha_{k,h}=1\big|y^\alpha_{k,h}=1)-\bP^{\pi^\beta_k,p^{*\beta}}(a^\beta_{k,h}=1|y^\beta_{k,h}=1)\big|.$$

Similarly, by triangular inequality, 

\begin{align*}
\cC_{reg}^{EqOpt}(k)&=\frac{1}{H}\sum_{h=1}^H \big|\bP^{\pi^{\alpha}_k,p^{*\alpha}}(a^\alpha_{k,h}=1\big|y^\alpha_{k,h}=1)-\bP^{\pi^\beta_k,p^{*\beta}}(a^\beta_{k,h}=1|y^\beta_{k,h}=1)\big|\\
&\le \frac{1}{H}\sum_{h=1}^H \big|\bP^{\pi^{\alpha}_k,p^{\alpha}_k}(a^\alpha_{k,h}=1\big|y^\alpha_{k,h}=1)-\bP^{\pi^\beta_k,p^{\beta}_k}(a^\beta_{k,h}=1|y^\beta_{k,h}=1)\big|\\
&+ \frac{1}{H}\sum_{h=1}^H\sum_{\vartheta} \big|\bP^{\pi^{\vartheta}_k,p^{\vartheta}_k}(a^\vartheta_{k,h}=1\big|y^\vartheta_{k,h}=1)-\bP^{\pi^{\vartheta}_k,p^{*\vartheta}}(a^\vartheta_{k,h}=1|y^\vartheta_{k,h}=1)\big|.
\end{align*}

Notice that switching $\pi^{*\vartheta}$ to $\pi^\vartheta_k$ doesn't change our argument in Lemma~\ref{lm:d}, thus with probability at least $1-2\delta$, for $\epsilon_k$ (we take $\epsilon_k$ to be consistent to $\hat{c}_{k,h}$),

$$\cC_{reg}^{EqOpt}(k)\le\left\{\begin{matrix}
	\sum_{\vartheta}\frac{6H\sqrt{\frac{2S\ln(32SAk^2/(\epsilon_k\delta))}{N^{\vartheta,min}_{k}}}+6\epsilon_k HS}{p^{\vartheta,min}_{k}\left(p^{\vartheta,min}_{k}-\sqrt{\frac{4\ln 2+2\ln(4SAk^2/\delta)}{N^{\vartheta,min}_{k}}}\right)} ,&\text{if}~p^{\vartheta,min}_{k}>\sqrt{\frac{4\ln 2+2\ln(4SAk^2/\delta)}{N^{\vartheta,min}_{k}}};\\ 
	1,&\text{otherwise}.
\end{matrix}\right.$$
Meanwhile, it also holds simultaneously that for all $(s,a)\in\cS\times \cA$,
$$|p_k(y=1|s,a)-p^*(y=1|s,a)|\le \sqrt{\frac{4\ln 2+2\ln(4SAk^2/\delta)}{N_k(s,a)}}.$$

Thus, if $4\sqrt{\frac{4\ln 2+2\ln(4SAk^2/\delta)}{\min_{s,a}N_k(s,a)}}<C$, we have 
$$\cC_{reg}^{EqOpt}(k)\le \sum_{\vartheta}\frac{6H\sqrt{\frac{2S\ln(32SAk^2/(\epsilon_k\delta))}{N^{\vartheta,min}_{k}}}+6\epsilon_k HS}{c^2/4}$$


Recall with probability at least $1-\delta$, for all $k\in[K]$, $(s,a)\in\cS\times\cA$, $\vartheta\in\{\alpha,\beta\}$
$$N^\vartheta_k(s,a)\ge \left(1-\sqrt{\frac{2H^2\ln(4k^2SA/\delta)}{C\eta_kn_k(H-1)(k-1)}}\right)C\eta_kn_k(H-1)(k-1).$$

Then, when $k\ge \frac{8H^2\ln(4k^2SA/\delta)}{C\eta_kn_k(H-1)}$,

$$N^\vartheta_k(s,a)\ge \frac{1}{2}C\eta_kn_k(H-1)(k-1). $$

Thus, we have the following properties:

\begin{itemize}
    \item With probability $1-3\delta$,  there exists a constant threshold $C$, for $k\ge C'+ \left(\frac{8H^2\ln(4SA/\delta)}{Cn_k(H-1)}\right)^{3}$,
    we have 
   $$\min_{s,a}N^\vartheta_k(s,a)\ge \frac{1}{2}C n_k(H-1)(k-1)k^{-1/3}. $$
   As a result, 
   $$\cC_{reg}^{DP}(k)\le \sum_\vartheta 2H\sqrt{\frac{4S\ln(16S^2HAk^3/\delta)}{C n_k(H-1)(k-1)k^{-1/3} }}+\frac{4}{k}.$$
   \item With probability $1-3\delta$,  there exists a constant threshold $C'$, for $$k\ge\max\left\{ C'+ \left(\frac{8H^2\ln(4SA/\delta)}{Cn_k(H-1)}\right)^{3}, \left(\frac{32\ln2+16\ln(4SA/\delta)}{C^3(H-1)n_k}\right)^3\right\},$$
   $$\cC_{reg}^{EqOpt}(k)\le \sum_{\vartheta}\frac{24H\sqrt{\frac{4S\ln(32S^2HAk^3/\delta)}{C n_k(H-1)(k-1)k^{-1/3}}}+24/k}{C^2}.$$

\end{itemize}

Noticing $n_k\ge 1$, we obtain the final result.
\end{proof}

\section{Further Details About Experiments}\label{app:experiment}


\subsection{Additional Optimizations}
\label{Optimizations}
For the formulation of our algorithm via occupancy measure, please refer to Appendix~\ref{app:occupancy}. We here describe additional formulations for the surrogate optimization. The formulation for surrogate optimization with different fairness notions are discussed below. 
\subsubsection{Demographic Parity Penalized Objective Surrogate}
For episode $k$, the optimization problem can be reformulated as:
$$\max_\rho \sum_{x,y,a,h,\vartheta}p_\vartheta\rho^\vartheta_k(x,y,a,h) \hat{r}^{\vartheta}_k(x,y,a) - \lambda \sum_{h,a}(\sum_{y,x} \rho_k^\alpha(x,y,a,h)-\sum_{y,x} \rho_k^\beta(x,y,a,h))^2$$
such that 
$$\forall \vartheta, x, t~~ \frac{\rho^\vartheta_k(x,y=1,a=1,h)}{\sum_{a}\rho^\vartheta_k(x,y=1,a,h)}=\frac{\rho^\vartheta_k(x,y=0,a=1,h)}{\sum_{a}\rho^\vartheta_k(x,y=0,a,h)}$$
$$\forall \vartheta, x',y',h\quad \sum_{a} \rho^\vartheta_k(x',y',a,h+1)=\sum_{x,y,a}\rho^\vartheta_k(x,y,a,h)p^\vartheta_{(k)}(x',y'|x,y,a),$$
$$\forall \vartheta,x,y,a,h  ~~0\le \rho^\vartheta_k(x,y,a,h)\le 1, \quad \sum_{x,y,a}\rho^\vartheta_k(x,y,a,h)=1.$$

\subsubsection{Equal Opportunity Penalized Objective Surrogate}
We implement change of variable techniques to convert the polynomial optimization problem to quadratic optimization problems for computational purpose. For episode $k$, the optimization problem can be reformulated as:
$$\max_{\rho,u,v} \sum_{x,y,a,h,\vartheta}p_\vartheta\rho^\vartheta_k(x,y,a,h) \hat{r}^{\vartheta}_k(x,y,a)-\sum_h \lambda (u_h - v_h)^2$$
such that 

$$\forall h~~ v_h=\sum_x\rho_k^\alpha(x,y=1,a=1,h)\sum_{x,a}\rho_k^\beta(x,y=1,a,h)$$
$$\forall h~~ u_h=\sum_x\rho_k^\beta(x,y=1,a=1,h)\sum_{x,a}\rho_k^\alpha(x,y=1,a,h)$$

$$\forall \vartheta, x, t~~ \frac{\rho^\vartheta_k(x,y=1,a=1,h)}{\sum_{a}\rho^\vartheta_k(x,y=1,a,h)}=\frac{\rho^\vartheta_k(x,y=0,a=1,h)}{\sum_{a}\rho^\vartheta_k(x,y=0,a,h)}$$

$$\forall \vartheta, x',y',h\quad \sum_{a} \rho^\vartheta_k(x',y',a,h+1)=\sum_{x,y,a}\rho^\vartheta_k(x,y,a,h)p^\vartheta_{(k)}(x',y'|x,y,a),$$

$$\forall \vartheta,x,y,a,h  ~~0\le \rho^\vartheta_k(x,y,a,h)\le 1, \quad \sum_{x,y,a}\rho^\vartheta_k(x,y,a,h)=1.$$

\subsection{Additional Results for Synthetic Data}
\subsubsection{Synthetic Data}
For the population dynamics,  
we model the intial qualification distribution as $\bP^\vartheta(y^{\vartheta}_0=1)$, and initial feature distribution conditioned on qualification
$\bP^\vartheta(x_0=j|y_0=w)$. For a loan setting, we can interpret a higher feature value $i$ as corresponding to a better credit score. 
Then, we generate the underlying group-independent and time invariant transition kernel $p^{*\vartheta}$ in the following way: we first set a distribution for $p^{*\vartheta}(y'=w'|y = w, a = v)$ for $w',w,v\in\{0,1\}$. 
Then, we set $p^{*\vartheta}(x'=j'|x=j,y'=w',a=v)$ for $j',j\in\cX$ and $w',v\in\{0,1\}$. 
Thus we set $p^{*\vartheta}(x',y'|x,y,a)$ as $p^{*\vartheta}(y'|y, a)$ $p^{*\vartheta}(x'|x,y',a)$

\subsubsection{Experimental Results}
Figure~\ref{DpSimGenerativeMix_linear_estConstr_policyUpdateAvgReward} shows the Pareto frontier in terms of episodic total return and episodic step-average fairness violation for demographic parity, and Figure \ref{EqOptSimGenerativeMix_linear_estConstr_policyUpdateAvgReward} shows the counterpart for equal opportunity.
Figure~\ref{DpSimGenerativeMix_linear_estConstr_trainEpisodeIndex} and~\ref{DpSimGenerativeMix_linear_policyUpdateAvgReward_trainEpisodeIndex} demonstrate the training dynamics of different algorithms
for demographic parity, and 
Figures~\ref{EqOptSimGenerativeMix_linear_estConstr_trainEpisodeIndex} and~\ref{EqOptSimGenerativeMix_linear_policyUpdateAvgReward_trainEpisodeIndex} demonstrate the counterpart
for equal opportunity.
Our proposed method outperforms the baselines in terms of Pareto frontiers, and converges to a stable level in terms of fairness violation over the training episodes. 
Overall, from the Pareto frontier, we can observe that our method reaches smaller fairness violation level for two groups of individuals in both fairness notions, under a fixed reward level for the decision maker. In addition, from the confidence intervals, we can see that our algorithm has a much narrower confidence band than the other baseline.

\begin{figure*}[h!]
\centering

\begin{subfigure}[b]{0.32\textwidth}
\centering
\includegraphics[width = \textwidth]{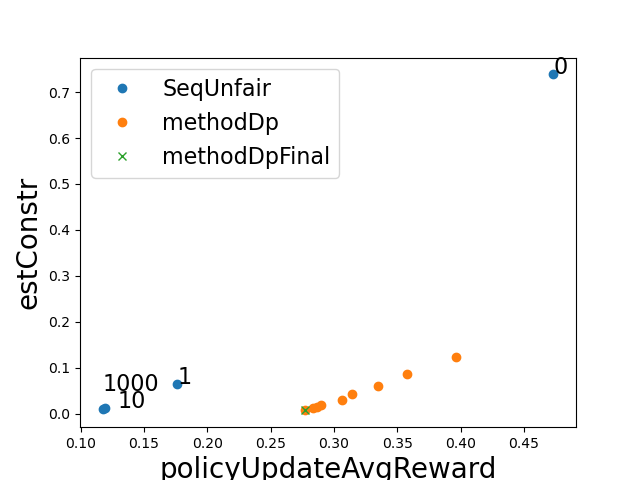}
\caption{DP Synthetic Pareto ~~~~~~~~~ \label{DpSimGenerativeMix_linear_estConstr_policyUpdateAvgReward}}
\end{subfigure}
\hfill
\begin{subfigure}[b]{0.32\textwidth}
\centering
\includegraphics[width = 1\textwidth]{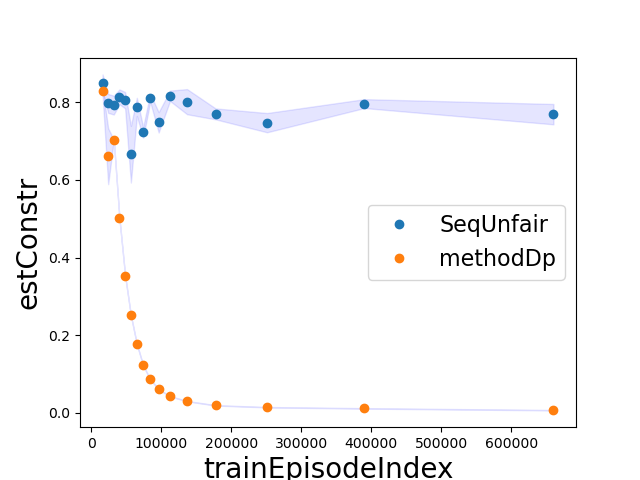}
\caption{DP Synthetic estConstr \label{DpSimGenerativeMix_linear_estConstr_trainEpisodeIndex}}
\end{subfigure}
\hfill
\begin{subfigure}[b]{0.32\textwidth}
\centering
\includegraphics[width = 1\textwidth]{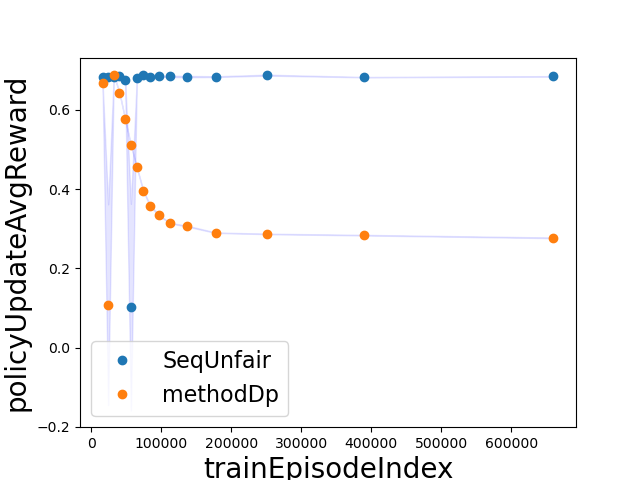}
\caption{DP Synthetic return \label{DpSimGenerativeMix_linear_policyUpdateAvgReward_trainEpisodeIndex}}
\end{subfigure}
\vskip\baselineskip
\begin{subfigure}[b]{0.32\textwidth}
\centering
\includegraphics[width = \textwidth]{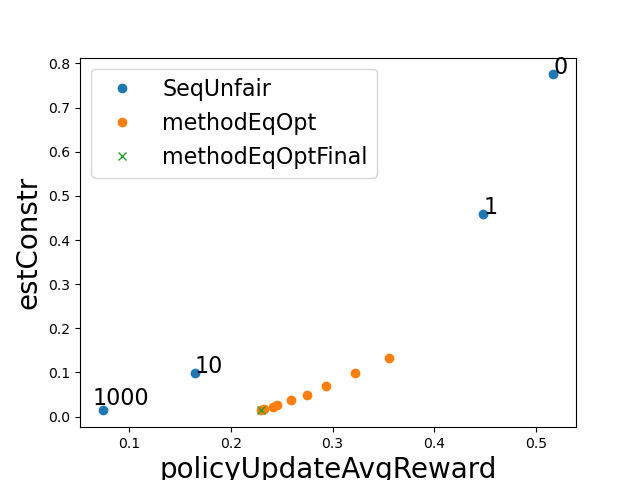}
\caption{EqOpt Synthetic Pareto ~~~~~~~~~ \label{EqOptSimGenerativeMix_linear_estConstr_policyUpdateAvgReward}}
\end{subfigure}
\hfill
\begin{subfigure}[b]{0.32\textwidth}
\centering
\includegraphics[width = 1\textwidth]{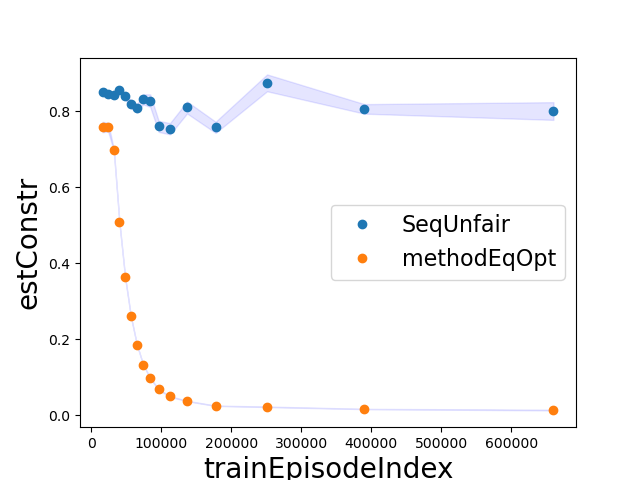}
\caption{EqOpt Synthetic estConstr ~~~\label{EqOptSimGenerativeMix_linear_estConstr_trainEpisodeIndex}}
\end{subfigure}
\hfill
\begin{subfigure}[b]{0.32\textwidth}
\centering
\includegraphics[width = 1\textwidth]{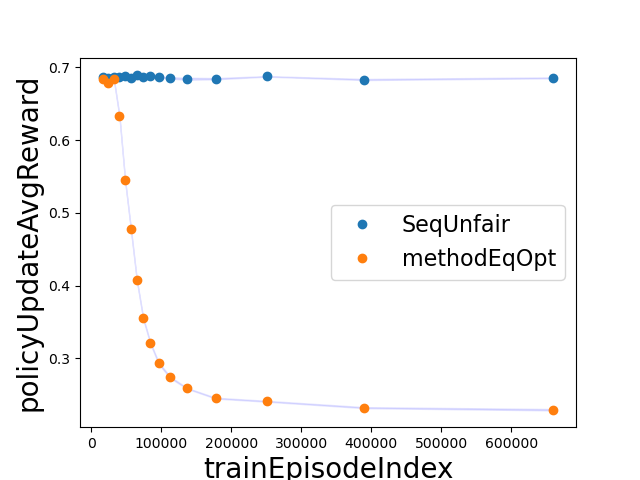}
\caption{EqOpt Synthetic return \label{EqOptSimGenerativeMix_linear_policyUpdateAvgReward_trainEpisodeIndex}}
\end{subfigure}
$$\vspace{-0.5in}$$
\caption{Synthetic data result. \ref{DpSimGenerativeMix_linear_estConstr_policyUpdateAvgReward} and \ref{EqOptSimGenerativeMix_linear_estConstr_policyUpdateAvgReward} for Pareto frontier, \ref{DpSimGenerativeMix_linear_estConstr_trainEpisodeIndex}\ref{DpSimGenerativeMix_linear_policyUpdateAvgReward_trainEpisodeIndex} for the constraint violation level in the training dynamics  \ref{EqOptSimGenerativeMix_linear_estConstr_trainEpisodeIndex} \ref{EqOptSimGenerativeMix_linear_policyUpdateAvgReward_trainEpisodeIndex} for the avergage episodic return over the training dynamics. In Pareto plots, the point with cross marker is the result for the final episode. The text close to points stand for penalty parameters}
\end{figure*}

\subsection{Detailed choice for parameters for experiments}
We discuss the choice of parameters for our data generating processes. 
\subsubsection{Synthetic Data}
\label{Synthetic Data}
We set $\bP(x_0=i|y_0=1) = 0.2$ and $\bP(x_0=i|y_0=0) = 0.2$ for initial probability conditioned on the qualification status of the individual. \\
We set $p^{*\vartheta}(y'=1|y = 1, a = 1) = 0.6$, $p^{*\vartheta}(y'=1|y = 1, a = 0) = 0.4$, $p^{*\vartheta}(y'=1|y = 0, a = 1) = 0.6$ and $p^{*\vartheta}(y'=1|y = 0, a = 0) = 0.4$ .\\
We define $d_{{\vartheta},j',j,w'} := p^{*\vartheta}(x'=j'|x = j, y' = w'), \forall \vartheta,j',j,w'$. 
We further define a vector: $D_{\vartheta,j,w'}=[d_{\vartheta,j',j,w'}]_{j'}$.
And we set $D_{\vartheta,0,w'} = [0.3, 0.25, 0.2, 0.15, 0.1]$, $D_{\vartheta,1,w'} = [0.22, 0.26, 0.22, 0.17, 0.13]$, $D_{\vartheta,2,w'} = [0.17, 0.21, 0.24, 0.21, 0.17]$, $D_{\vartheta,3,w'} = [0.13, 0.17, 0.22, 0.26, 0.22]
$ and $D_{\vartheta,4,w'} = [0.1, 0.15, 0.2, 0.25, 0.3], \forall \vartheta, w'$\\
Here, we use asymmetric $p^{*\vartheta}$ value such that we have higher probability to obtain sampled individual who is qualified for the next time step if we give positive decision at this step.\\
This design purpose is we would like to reflect the fact that the sampled individual will be motivated by positive decision from decision maker and demotivated by negative decisions. 

\subsubsection{Semi-Realistic Data}
\label{Semi-Realistic Data}
Notice our experiments in the main context on FICO data is semi-realistic data. We here include the details about the full data generating process. We define 
$g_{\vartheta,j',j, 1, 1} := p^{*\vartheta}(x'=j'|x = j, y = 1, a = 1)$, $g_{\vartheta,j',j, 1, 0} := p^{*\vartheta}(x'=j'|x = j, y = 1, a = 0)$, $g_{\vartheta,j',j, 0, 1} := p^{*\vartheta}(x'=j'|x = j, y = 0, a = 1)$ and $g_{\vartheta,j',j, 0, 0} := p^{*\vartheta}(x'=j'|x = j, y = 0, a = 0)$. 
We further define a vector: $G_{\vartheta,j,w,v}=[g_{\vartheta,j',j,w,v}]_{j'}$, 
and we set,\\
$G_{\vartheta,0, 1, 1} = [0.3, 0.25, 0.2, 0.15, 0.1]$, $G_{\vartheta,1, 1, 1} = [0.18, 0.27, 0.23, 0.18, 0.14]$, $G_{\vartheta,2, 1, 1} = [0.14, 0.18, 0.27, 0.23, 0.18]$, $G_{\vartheta,3, 1, 1} = [0.1, 0.15, 0.2, 0.3, 0.25]$, $G_{\vartheta,4, 1, 1} = [0.06, 0.13, 0.19, 0.24, 0.38]$, \\
$G_{\vartheta,0, 1, 0} = [0.38, 0.24, 0.19, 0.13, 0.06]$, $G_{\vartheta,1, 1, 0} = [0.25, 0.3, 0.2, 0.15, 0.1]$, $G_{\vartheta,2, 1, 0} = [0.18, 0.23, 0.27, 0.18, 0.14]$, $G_{\vartheta,3, 1, 0} = [0.14, 0.18, 0.23, 0.27, 0.18]
$, $G_{\vartheta,4, 1, 0} = [0.1, 0.15, 0.2, 0.25, 0.3]$,\\
$G_{\vartheta,0, 0, 1} = [0.3, 0.25, 0.2, 0.15, 0.1]$, $G_{\vartheta,1, 0, 1} = [0.18, 0.27, 0.23, 0.18, 0.14]$, $G_{\vartheta,2, 0, 1} = [0.14, 0.18, 0.27, 0.23, 0.18]$, $G_{\vartheta,3, 0, 1} = [0.1, 0.15, 0.2, 0.3, 0.25]
$, $G_{\vartheta,4, 0, 1} = [0.1, 0.15, 0.2, 0.25, 0.3]$,\\
$G_{\vartheta,0, 0, 0} = [0.38, 0.24, 0.19, 0.13, 0.06]$, $G_{\vartheta,1, 0, 0} = [0.25, 0.3, 0.2, 0.15, 0.1]$, $G_{\vartheta,2, 0, 0} = [0.18, 0.23, 0.27, 0.18, 0.14]$, $G_{\vartheta,3, 0, 0} = [0.14, 0.18, 0.23, 0.27, 0.18]
$, $G_{\vartheta,4, 0, 0} = [0.1, 0.15, 0.2, 0.25, 0.3]$

Similar to synthetic data generating process, $g_{\vartheta,j',j,w,v}$ are set such that $x$ has higher probability to transition to a higher value of $x'$ for the next step when we make a positive decision at the current step, and lower probability to transition to a lower value of $x'$ for the next step when we make a negative decision at the current step.\\
For example, when $x = 2$, we have higher probability that $x$ will transition to $x' = 3$ than $x' = 1$ given $a=1$. \\

\end{document}